%% file: arXiv.tex
\RequirePackage{fix-cm}
\documentclass[twoside,english,12pt]{article}
\usepackage[T1]{fontenc}
\usepackage[latin9]{inputenc}
\usepackage{geometry}
\usepackage{color}
\usepackage{babel}
\usepackage{verbatim}
\usepackage{marvosym}
\usepackage{bm}
\RequirePackage{amsmath}
\RequirePackage{amsthm}
\RequirePackage{amssymb}
\RequirePackage[sort&compress,numbers]{natbib}
\usepackage[unicode=true,pdfusetitle,
 bookmarks=true,bookmarksnumbered=false,bookmarksopen=false,
 breaklinks=false,pdfborder={0 0 1},backref=false,colorlinks=true]
 {hyperref}
\hypersetup{
 citecolor=red, linkcolor=blue}

\usepackage[framemethod=tikz]{mdframed}

\makeatletter

\theoremstyle{plain}
\newtheorem{thm}{\protect\theoremname}
\numberwithin{thm}{section}
\theoremstyle{plain}
\newtheorem{cor}{\protect\corollaryname}  
\numberwithin{cor}{section}
\theoremstyle{definition}
\newtheorem{defn}{\protect\definitionname}
\numberwithin{defn}{section}
\theoremstyle{remark}
\newtheorem{exmp}{\protect\examplename}
\numberwithin{exmp}{section}

\numberwithin{equation}{section}

\usepackage{microtype, lmodern}
\usepackage[bb=boondox]{mathalfa}
\RequirePackage{graphicx}

\makeatother

  \providecommand{\definitionname}{Definition}
\providecommand{\corollaryname}{Corollary}
\providecommand{\theoremname}{Theorem}
 \providecommand{\examplename}{Example}
 
\title{Solving Equations of Random Convex Functions\\ via Anchored Regression}
\author{Sohail~Bahmani\\\texttt{sohail.bahmani@ece.gatech.edu}						\and Justin~Romberg\\\texttt{jrom@ece.gatech.edu}}  

\date{\today}
\input{./macros.tex}

\begin{document}

\maketitle
\begin{abstract}
We consider the question of estimating a solution to a system of equations that involve convex nonlinearities, a problem that is common in machine learning and signal processing.  Because of these nonlinearities, conventional estimators based on empirical risk minimization generally involve solving a non-convex optimization program.  We propose \emph{anchored regression}, a new approach based on convex programming that amounts to maximizing a linear functional (perhaps augmented by a regularizer) over a convex set.   The proposed convex program is formulated in the natural space of the problem, and avoids the introduction of auxiliary variables, making it computationally favorable.  Working in the native space also provides great flexibility as structural priors (e.g., sparsity) can be seamlessly incorporated. 

For our analysis, we model the equations as being drawn from a fixed set according to a probability law.  Our main results  provide guarantees on the accuracy of the estimator in terms of the number of equations we are solving, the amount of noise present, a measure of statistical complexity of the random equations, and the geometry of the regularizer at the true solution.  We also provide recipes for constructing the anchor vector (that determines the linear functional to maximize) directly from the observed data.
\end{abstract}


\section{Introduction}
\label{sec:intro}

We consider the problem of (approximately) solving a system of nonlinear
equations with convex nonlinearities. In particular, we observe 
\begin{equation}
\begin{aligned}y_{1} & =f_{1}\left(\mb x_{\star}\right)+\xi_{1}\\
y_{2} & =f_{2}\left(\mb x_{\star}\right)+\xi_{2}\\
 & \vdots\\
y_{M} & =f_{M}\left(\mb x_{\star}\right)+\xi_{M},
\end{aligned}
\label{eq:convex-equations}
\end{equation}
where $\mb x_{\star}\in\mbb R^{N}$ contains the ground-truth parameters to be estimated, the functions $f_{1},\,f_{2},\,\ldots,\,f_{M}$ are convex and known, 
and the $\xi_{1},\,\xi_{2},\,\ldots,\,\xi_{M}$ are additive
noise terms. For simplicity, the functions $f_m$ are also assumed to be differentiable throughout the paper.
%
%
Given the observations \eqref{eq:convex-equations}, we propose as an estimator for
$\mb x_{\star}$ the solution of a convex program that balances consistency with the observations
against structure induced by a (convex) regularization term. 
This convex formulation means that if the functions $f_m$ and their first few derivatives can be computed efficiently, the proposed estimator is computationally tractable.

Our main results give error bounds on the quality of the produced estimate, specified by the Euclidean distance to $\mb x_{\star}$, under a model where the functions $f_{m}$ are drawn at random from a set $\mc F$ according to some probability law. 
The estimation error, and the sufficient number of equations $M$
needed to achieve it, depend on the \emph{Rademacher complexity} of a set $\mc A$ with respect to the induced probability law on the gradients $\nabla f_m$ --- $\mc A$ is essentially the set of all ascent directions of the functional being maximized at $\mb x_{\star}$, and depends on the geometry of the regularizer.

One of these results, Corollary~\ref{cor:unstruct}, shows that if the $f_m$ are ``generic'' convex functions (in that their gradients are sufficiently diverse in $\mathbb{R}^N$), then $\mb{x}_\star$ can be recovered from $M\sim N$ observations $y_m=f_m(\mb{x}_\star)$.

\subsection{Motivating examples}
\label{ssec:examples}

To illustrate the broad applicability of the observation model \eqref{eq:convex-equations}, we provide
a few motivating examples. 
\begin{exmp}[Convex function of linear predictor] Perhaps the simplest example is the case that
\begin{align}
f_m(\mb{x})&= \phi(\mb{a}_m^{\T} \mb{x}) & m=1,2,\dotsc,M \label{eq:linear-pred}\,,
\end{align}
for a given convex function $\phi :\mbb{R}\to\mbb{R}$, and data samples $\mb{a}_1,\mb{a}_2,\dotsc,\mb{a}_M \in \mbb{R}^N$. This form of observation appears in many instances of \emph{generalized linear models} (GLM)
\citep{McCullagh-Generalized-1989} with $\phi$ being
the \emph{mean function }(or the \emph{inverse link function}) of the GLM. In some cases, the mean function may not be convex itself, but it has a property (e.g., it is concave,
log-convex, or log-concave) that allows us to easily convert it to a convex function. In such cases, $\phi$ would be the appropriate transformation of the mean function.

Equations of the form \eqref{eq:linear-pred} also appear in the \emph{phase retrieval} problem from computational imaging. The goal in phase retrieval is to estimate
an signal/image (up to a global phase) from intensity of some linear
observations. More precisely, given (random) \emph{measurement vectors}
$\mb a_{1},\mb a_{2},\dotsc,\mb a_{M}\in\mbb C^{N}$ the goal is
to estimate $\mb x_{\star}\in\mbb C^{N}$ from noisy intensity measurements
$y_{m}=\left|\mb a_{m}^{*}\mb x_{\star}\right|^{2}+\xi_{m}$ for $m=1,2,\dotsc,M$.
Aside from minor technicalities for treating complex-valued vectors, clearly the observation model in phase retrieval is a special case of \eqref{eq:linear-pred}
with $\phi(z)=\left|z\right|^{2}$
being the nonlinearity.
\end{exmp}
\begin{exmp}[A simple neural network] The second example is related to learning with a\emph{
single hidden-layer neural network} \cite{Haykin-Neural-2009,Shalev_Shwartz-Understanding-2014}.
Given prescribed
weights $w_{1},w_{2},\dotsc,w_{K}\ge0$ for the hidden layer of a single hidden-layer neural network, and an activation function $\phi:\mbb R\to\mbb R$, the
output  of the network can be expressed as\looseness=-1
\begin{equation}
	\label{eq:exNN}
	f_{m}\left(\mb x\right)=\sum_{k=1}^{K}w_{k}\phi\left(\mb a_{m}^{\T}\mb x_{k}\right)\,,
\end{equation}
where $\mb x=\left[\begin{array}{cccc}
\mb x_{1}^{\T} & \mb x_{2}^{\T} & \dotsm & \mb x_{K}^{\T}\end{array}\right]^{\T}$ denotes the input layer's weight parameters and $\mb a_{m}$ denotes
the $m$th data sample. For convex activation functions such as the popular $\phi(u) = \max(u,0)$  and $\phi(u)=\log_{2}\left(1+e^{u}\right)$,
the functions $f_{m}\left(\mb x\right)$ are also convex and the model
assumed in \eqref{eq:convex-equations} applies.

The assumption of the knowledge of the output weights can be dropped, if the activation function is also positive homogeneous. In fact, each non-negative weight $w_k$ of the output layer can be absorbed into the corresponding input layer weight $\mb{x}_k$ because they interact only through $w_k \phi(\mb{a}_m^\T \mb{x}_k)=\phi\left(\mb{a}_m^\T(w_k\mb{x}_k)\right)$. Therefore,  this case reduces to the previous example with $w_1=w_2=\cdots=w_K=1$. We may need to adapt the regularizer on the $w_k$ and $\mb{x}_k$ to account for this combination of variables.
\end{exmp}

\begin{exmp}[Generalized neural nets and dimensionality reduction]
As an extension of the example above, we can consider the problem of recovering a $N\times K$ ($K<N$)  matrix $\mb{X}$ from observations of the form 
\begin{equation}
	\label{eq:exgenNN}
	f_{m}\left(\mb{X}\right)=\sum_{t=1}^T w_t\mu_t\left(\mb{X}^\T\mb a_{m}\right),
\end{equation}
for fixed functions $\mu_1,\mu_2,\dotsc,\mu_T:\mbb R^{K}\to\mbb R$ and weights $w_1,w_2,\dotsc,w_T$.  We might interpret problems of this form in two different ways.  First, we can view \eqref{eq:exgenNN} as a generalization of the single-layer neural net in \eqref{eq:exNN}, where now the nonlinearties $\mu_t$ are acting {\em jointly}, rather than pointwise, across the output of the linear layer.  Another interpretation is that we are trying to find a regression function that works on feature vectors in a reduced dimension space, replacing $\mb{a}_m\in\mbb R^N$ with $\mb{X}^\T\mb{a}_m\in\mbb{R}^K$, and we jointly learn how to choose this linear embedding and the weights $w_t$.


By naturally flattening $\mb{X}$ into a vector in $\mbb{R}^{NK}$, and assuming that the functions $\mu_t\left(\cdot\right)$ are convex with  known weights $w_t\ge 0$, our framework applies to this model. Similar to the previous example, we may also relax the knowledge of $w_t$'s assuming that $\mu_t$'s are positive homogeneous so that with $\mb{w}=\left[\begin{array}{cccc}
w_1 & w_2 & \dotsm  & w_T
\end{array}\right]^\T$ we can write 
\[f_m(\mb{w}, \mb{X}) = \sum_{t=1}^T \mu_t (w_t \mb{X}^\T \mb{a}_m)\, .\] Again we can flatten $\mb{X}$ to $\mb{x}\in \mbb{R}^{NK}$ and observe that the functions $f_m(\mb{w},\mb{X})$ only depend on the rank-one matrix $\mb{Z} = \mb{x}\mb{w}^\T$. Namely, we can write 
\[f_m(\mb{w},\mb{X}) = \sum_{t=1}^T \mu_t\left(\left(\mb{I}_{K\times K}\otimes \mb{a}_m^\T \right)\mb{z}_t\right)\, , \]
where  $\otimes$ denotes the Kronecker product and $\mb{z}_t$ is the $t$th column of $\mb{Z}$. 
As a function of $\mb{Z}$, the observation functions $f_m(\mb{Z})$ are convex and compatible with the observation model \eqref{eq:convex-equations}. The only remaining concern is inducing the rank-one structure of $\mb{Z}$ which can be done through \emph{nuclear~norm} regularization in the estimator.
\end{exmp}

Some of the above examples can be extended further to address composition of non-negative mixtures of positive homogeneous functions by lifting the unknown parameters to a higher order low-rank tensor. However, estimating low-rank tensors usually involves computationally prohibitive operations which renders these possible extensions less interesting from a practical point of view.  

\subsection{Anchored regression}

In this paper, we propose \emph{anchored regression} to estimate $\mb x_{\star}$
in the parametric model described by \eqref{eq:convex-equations}.
An \emph{anchor vector} $\mb a_{0}\in\mbb R^{N}$ is a unit vector
(i.e., $\left\lVert \mb a_{0}\right\rVert _{2}=1$) that obeys
\begin{align}
\left\langle \mb a_{0},\mb x_{\star}\right\rangle  & \ge\delta\left\lVert \mb x_{\star}\right\rVert _{2},\label{eq:anchor}
\end{align}
 for an absolute constant $\delta\in(0,1]$. In words, the anchor
vector has a non-vanishing correlation with a ground truth. Given
an anchor vector $\mb a_{0}$, our proposed estimator for \eqref{eq:convex-equations}
is the convex program 
\begin{equation}
\begin{aligned}\argmax_{\mb x}\  & \left\langle \mb a_{0},\mb x\right\rangle -\Omega\left(\mb x\right)\\
\mr{subject\ to\ } & R_{M}^{\,+}\left(\mb x\right)\le R_{M}^{\,+}\left(\mb x_{\star}\right)+\varepsilon\,,
\end{aligned}
\label{eq:anchored-regression}
\end{equation}
where $\Omega\left(\mb x\right)$ is a convex regularizer, $\varepsilon>0$
is a small constant, and $R_{M}^{\,+}\left(\cdot\right)$ is the empirical
one-sided error
\begin{align*}
R_{M}^{\,+}\left(\mb x\right) & \defeq\frac{1}{M}\sum_{m=1}^{M}\left(f_{m}\left(\mb x\right)-y_{m}\right)_{+}\,,
\quad \left(\cdot\right)_{+} = \max(\cdot,0),
\end{align*}
 which is also convex. 
Note that when there is no noise in the observations and we take $\varepsilon=0$, the constraint $R_{M}^{\,+}\left(\mb x\right)\leq 0$ is equivalent to $f_m(\mb x)\leq y_m$ for all $m$.  
Of course, when there is noise, the value of 
\[R_{M}^{+}\left(\mb x_{\star}\right)=\frac{1}{M}\sum_{m=1}^M(-\xi_m)_+\,,\]
is unknown in general, but depending on the noise model we may assume $R_{M}^{+}\left(\mb x_{\star}\right)$ can be estimated and the absolute error of such an estimate is captured by $\varepsilon$ in \eqref{eq:anchored-regression}.

The anchor vector $\mb{a}_0$ can be interpreted as a ``rough guess'' for the solution $\mb{x}_\star$.  This guess might come from some kind of a priori information about the solution, or it might be formed directly from the observations.  In Section \ref{sec:recipes} below, we describe some schemes by which an $\mb{a}_0$ obeying \eqref{eq:anchor} can be constructed in a data-driven manner.

\subsection{Advantages of anchored regression}

Solving the system of nonlinear equations \eqref{eq:convex-equations} is in general a hard problem.  In the particular case where the $f_m$ are multivariate polynomials,  a standard approach is to \emph{lift} the variables to a higher-dimensional space that linearizes the equations.  The key idea is that any homogeneous polynomial function of the vector $\mb{x}\in\mathbb{R}^N$ of degree $d$ can be expressed as a linear function of the tensor $\mb{X} = \mb{x}^{\otimes d}$.  In the lifted domain, then, solving \eqref{eq:convex-equations} amounts to finding a \emph{rank-one tensor} solution to a system of linear equations.  For quadratic polynomials ($d=2$), there is a natural convex relaxation to this problem using semi-definite programming; the effectiveness of this relaxation for solving generic quadratic equations was studied in \cite{candes14so}, and analysis for several types of structured equations that arise in engineering problems can be found in \cite{candes13ph,ahmed14bl,ling15se}.  For $d\ge3$, it is not clear how the lifted problem can be relaxed, as the required tensor computational primitives are computationally hard \cite{Hillar-Tensor-2013}.  Even in the $d=2$ case where the relaxation can be cast as an SDP, the computational cost in squaring the number of variables can be prohibitive for medium- to large-scale problems.

In contrast, our estimator works for observations that are general convex functions of the unknown variables.\footnote{Of course, we need to be able to evaluate the $f_m$ and some number of its derivatives to actually solve \eqref{eq:anchored-regression}.}  Our method works in the natural domain of the unknown variables, which is computationally efficient and provides us with the flexibility to incorporate prior structural information about the ground truth through regularization.

The anchor $\mb{a}_0$ also allows us to avoid one of the pitfalls of convexification.  It can be that there are multiple equally valid solutions to \eqref{eq:convex-equations}, especially since the $f_m$ are convex.  For instance, in the phase retrieval problem where $f_m(\mb{x}) = |\mb{a}_m^*\mb{x}|^2$ for a set of complex-valued vectors $\mb{a}_m$, a global change of phase to $\mb{x}_\star$ will not affect the measurements at all.  Minimizing a convex loss function that assigns an equal score to each of these equivalent solutions is bound to fail, as all of the points in the convex hull of the equivalence set will score equally well or better.  The lifting approach for quadratic equations gets around this by mapping every equivalent solution to the same point in the lifted space.  Our method also handles this in a straightforward way by introducing a bias towards the solutions best aligned with the anchor $\mb{a}_0$.

\subsection{Geometry of convex equations}

The analysis of  \eqref{eq:anchored-regression} for solving \eqref{eq:convex-equations} has a very clean geometric interpretation.  Suppose that the observations are noise-free, $y_m=f_m(\mb{x}_\star)$, and we solve \eqref{eq:anchored-regression} without the regularization term and with $\varepsilon=0$.  Since $R_M^+(\mb{x}_\star)=0$, this is equivalent to solving 
\begin{equation}
	\label{eq:noisefree}
	\begin{aligned}
	\maximize_{\mb{x}\in\mathbb{R}^N}\ &\langle\mb{x},\mb{a}_0\rangle & \\
	\text{subject to}\ &f_m(\mb{x})\leq y_m, & m=1,\ldots,M\,.
	\end{aligned}
\end{equation}
Since all of the constraints in the program above are active at $\mb{x}_\star$, the KKT conditions for this program tell us that $\mb{x}_\star$ is indeed the solution when
\[
	-\mb{a}_0 + \sum_{m=1}^M\lambda_m\nabla f_m(\mb{x}_\star) = \mb{0},
\]
for some $\lambda_1,\ldots,\lambda_M\geq 0$.  More succinctly, \eqref{eq:noisefree} is successful when
\[
	\mb{a}_0 \in \operatorname{cone}\left(\{\nabla f_m(\mb{x}_\star)\}_{m=1}^M\right).
\]
This is illustrated in Figure~\ref{fig:nonlingeom}.  If we knew the $\nabla f_m(\mb{x}_\star)$ in addition to the $f(\mb{x}_\star)$, then generating such a $\mb{a}_0$ would be straightforward.  But in general, we do not have knowledge of these gradients.

\begin{figure}[h]
	\begin{center}
		\begin{tabular}{ccc}
			\includegraphics[height=2.4in]{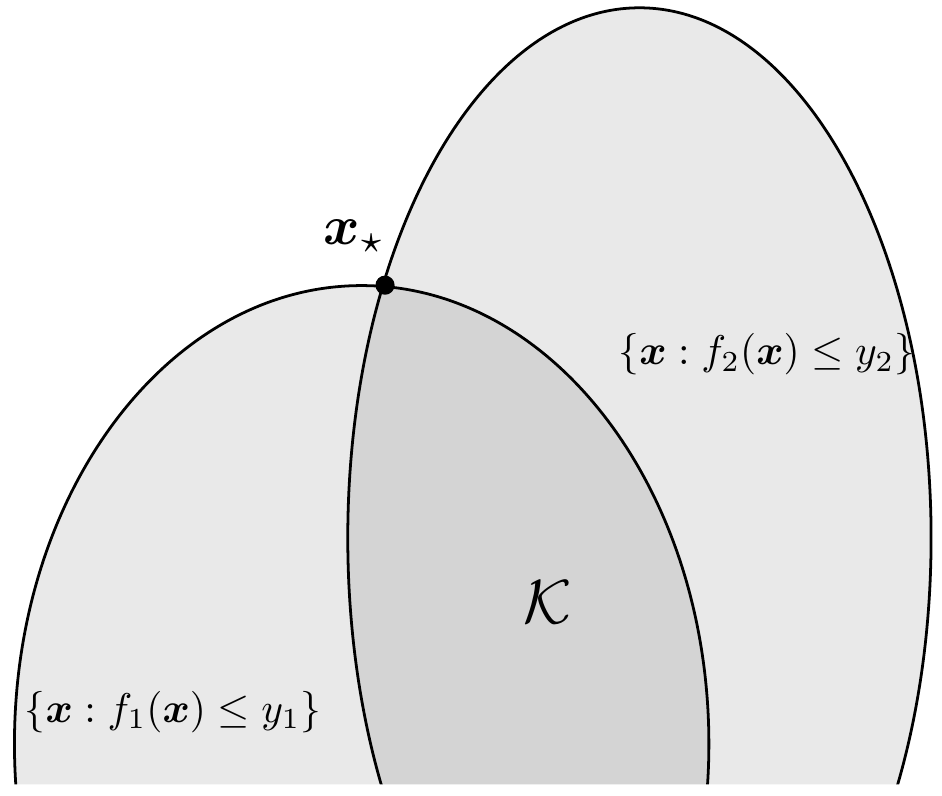} &
			\hspace{.5in} & 
			\includegraphics[height=2.25in]{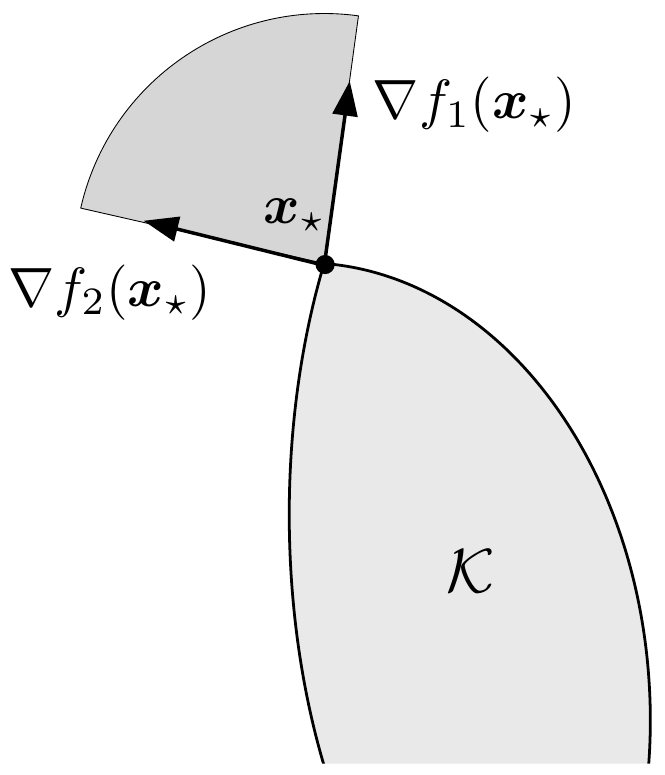} \\
			(a) & & (b) 
		\end{tabular}
	\end{center}
	\caption{\small\sl A simple set of two nonlinear equations, $y_1=f_1(\mb{x}_\star)$ and $y_2=f_2(\mb{x}_\star)$, in $\mathbb{R}^2$.  (a)  The unknown $\mb{x}_\star$ is an extreme point of $\mathcal{K} = \{\mb{x}~:~f_m(\mb{x})\leq y_m,~m=1,\ldots,M\}$.  (b) The program \eqref{eq:noisefree} will recover $\mb{x}_\star$ when $\mb{a}_0$ is in the cone generated by the two gradients.}
	\label{fig:nonlingeom}
\end{figure}

The main results in this paper say that it is enough to find a $\mb{a}_0$ that is roughly aligned with $\mb{x}_\star$.  That is, if the $f_m$ (and hence the $\nabla f_m(\mb{x}^\star)$) are generated at random, then with high probability
\[
	\langle\mb{x}_\star,\mb{a}_0\rangle\geq \delta \quad\Rightarrow\quad 
	\mb{a}_0\in\operatorname{cone}\left(\{\nabla f_m(\mb{x}_\star)\}_{m=1}^M\right),
\]  
provided that the number of equations $M$ is large enough.  In fact, we show that {\em every} roughly aligned vector will work.  Geometrically, this means that the cone of all valid anchor vectors
\[
	\mathcal{C}_\delta = \left\{\mb{z}~:~\langle\mb{z},\mb{x}_\star\rangle\geq\delta\|\mb{x}_\star\|_2\|\mb{z}\|_2\right\},
\]
is included in the cone generated by the gradients
\[
	\mathcal{C}_\delta\subseteq \operatorname{cone}\left(\{\nabla f_m(\mb{x}_\star)\}_{m=1}^M\right),
\]
again when $M$ is large enough.  As $\delta$ becomes larger (meaning that the anchor is more aligned with the true solution) the cone $\mathcal{C}_\delta$ shrinks, and the inclusion above may be satisfied with smaller $M$.

Adding a convex regularizer $\Omega(\mb{x})$ to the program makes the optimality conditions weaker.  If we again observe $y_m=f_m(\mb{x}_\star)$ and then solve
\[
	\begin{aligned}
	\maximize_{\mb{x}\in\mathbb{R}^N}\ & \langle\mb{a}_0,\mb{x}\rangle - \Omega(\mb{x}) & \\
	\text{subject to}\ & f_m(\mb{x})\leq y_m, &m=1,2,\dotsc,M\,,
	\end{aligned}
\]
then $\mb{x}_\star$ is the solution when
\[
	\mb{a}_0\in\operatorname{cone}\left(\{\nabla f_m(\mb{x}^\star)\}_{m=1}^M\right) + \partial\Omega(\mb{x}^\star),
\]
where $\partial\Omega(\mb{x}_\star)$ is the subdifferential of $\Omega$ at $\mb{x}_\star$,
\[
	\partial\Omega(\mb{x}) = \left\{\mb{g}~:~\Omega(\mb{x}+\mb{z})\geq\Omega(\mb{x}) + \mb{g}^\T\mb{z}~\text{for all}~\mb{z}\in\mathbb{R}^N\right\}.
\]

  Finding such a $\mb{a}_0$ is now easier than in the unregularized case, as the convex cone that needs to enfold it is larger.  The (subdifferential of the) regularization term is effectively working as an additional set of observations --- if the subgradients in $\partial\Omega(\mb{x}^\star)$ are not aligned with the $\nabla f_m(\mb{x}^\star)$, we may be able to get away with many fewer equations.

\subsection{Related work}

As mentioned in Section \ref{ssec:examples} our approach can be applied to most GLMs. Another relevant model is the
semiparametric \emph{single index model} (see  \cite{Ichimura-Semiparametric-1993}, for example)
where, again, the model assumes linear predictors, but the nonlinear
function $\mu$ is not known. Under some regularity assumption on $\mu$, in \citep{Plan-Generalized-2016,Plan-High_dimensional-2016} simple estimators based on convex programming are shown  to produce
accurate estimates (up to some scaling factor). Because $\mu$ is unknown, our framework
does not apply to the single index model. However, there are
interesting instances of \eqref{eq:convex-equations} (e.g., phase
retrieval) where the results of \citep{Plan-Generalized-2016,Plan-High_dimensional-2016} do not apply
as the assumed regularity conditions lead to trivial bounds.

The most relevant results to our work are the recent methods proposed for phase retrieval in \citep{Bahmani-Phase-2016} and independently
in \citep{Goldstein-PhaseMax-2016} that exploit anchor vectors. As mentioned above, the phase retrieval problem can be described by the
model \eqref{eq:convex-equations} with $f_{m}\left(\mb x\right)=\left|\mb a_{m}^{*}\mb x\right|^{2}$
for some random measurement vector $\mb a_{m}$. It is shown in \citep{Bahmani-Phase-2016,Goldstein-PhaseMax-2016}
that, using an anchor, the ground truth $\mb x_{\star}$ can be estimated accurately and with optimal sample complexity 
through a convex program analogous to \eqref{eq:anchored-regression}. The analyses in \citep{Goldstein-PhaseMax-2016}
and \citep{Bahmani-Phase-2016} differ in that the former assumes
the anchor is independent of the measurements whereas the latter does
not make this assumption. Alternative proofs and variations of this
phase retrieval method also appeared later in \citep{Hand-Elementary-2016,Hand-Compressed-2016,Hand-Corruption-2016}.

Our results for systems of equations with structured solutions are related to work on nonlinear compressed sensing \cite{bahmani13gr,blumensath13co,beck13sp,shechtman14ge,ehler14qu}.  In contrast to these works, our treatment takes place in a more general setting.  Theorem~\ref{thm:main-theorem} and Corollary~\ref{cor:main} give bounds on the number of equations needed to estimate $\mb x_{\star}$ to a certain accuracy that are based on general properties of the regularizer at the solution and the gradients of the functions $f_m$ --- they are not restricted to sparse solutions to nonlinear systems with highly specialized structure.

Recently, gradient descent methods for two particular nonlinear problems of interest in signal processing and machine learning, phase retrieval as mentioned above and ReLU regression, have been analyzed in the literature \cite{Candes_Phase_2014,Soltanolkotabi-Learning-2017}.  The specialization of our anchored regression technique to the phase retrieval problem is thoroughly detailed in \cite{Bahmani-Phase-2016,bahmani17fl}, and we will discuss its application to the ReLU regression problem in Section~\ref{ssec:RadCom} below.

\section{Main result}
\label{sec:main}

To show that \eqref{eq:anchored-regression}
produces an accurate estimate of $\mb x_{\star}$, it suffices to
show that the set of ascent directions,
\begin{align*}
\mc A & =\left\{ \mb h\,:\,\left\langle \mb a_{0}-\mb g,\mb h\right\rangle \ge0,\mr{\ for\ all\ }\mb g\in\partial\Omega\left(\mb x_{\star}\right)\right\} \,,
\end{align*}
 does not contain any vector with large $\ell_{2}$ norm that is consistent
with the constraint. Namely, if $\mb{h}\in\mc{A}$ has a large $\ell_2$ norm, then $R_M^+(\mb{x}_\star+\mb{h})>R^+_M(\mb{x}_\star) +\varepsilon$.

As explained further in Section \ref{sec:recipes}, we will be interested in anchor vectors 
$\mb a_{0}$ constructed from the observations $y_m$ (and knowledge of the $f_m$).  However, the dependence of the set $\mc A$ on the observations would complicate the probabilistic analysis.  We avoid this dependence by an expansion of the set $\mc A$ that allows us to decouple it from $\mb a_{0}$.

Let  $\mb h_{\perp}$ and ${\mb{a}_0}_\perp$ denote the projection of $\mb h$ and $\mb{a}_0$ onto the hyperplane orthogonal to $\mb x_{\star}$, respectively. Using the assumed property \eqref{eq:anchor} of the unit-norm anchor vector $\mb{a}_0$,  we can write 
\begin{align*}
\left\langle \mb a_{0},\mb h\right\rangle  & =\langle{\mb a_{0}}_\perp,\mb h\rangle+\langle\frac{\mb x_{\star}\mb{x}_\star^\T\mb{a}_0}{\left\lVert \mb x_{\star}\right\rVert^2 _{2}},\mb h\rangle\\
 & \le\left\lVert {\mb a_{0}}_\perp\right\rVert _{2}\left\lVert \mb h_\perp\right\rVert _{2}+\left(\frac{\mb{x}_\star^\T\mb{a}_0}{\left\lVert\mb{x}_\star\right\rVert_2}-\delta\right)\langle\frac{ \mb{x}_\star}{\left\lVert\mb{x}_\star\right\rVert_2},\mb h\rangle + \delta\langle\frac{\mb x_{\star}}{\left\lVert \mb x_{\star}\right\rVert _{2}},\mb h\rangle\\
 & \le\sqrt{\left\lVert {\mb a_{0}}_\perp\right\rVert _2^2+\left(\frac{\mb{x}_\star^\T\mb{a}_0}{\left\lVert\mb{x}_\star\right\rVert_2}-\delta\right)^2}\sqrt{\left\lVert \mb h_\perp\right\rVert^2_2+\langle\frac{ \mb{x}_\star}{\left\lVert\mb{x}_\star\right\rVert_2},\mb h\rangle^2}+\delta\langle\frac{\mb x_{\star}}{\left\lVert \mb x_{\star}\right\rVert _{2}},\mb h\rangle\\
 & \le \sqrt{1-\delta^2}\left\lVert\mb{h}\right\rVert_2 + \delta \langle\frac{\mb x_{\star}}{\left\lVert \mb x_{\star}\right\rVert _{2}},\mb h\rangle\,,
\end{align*}
The second line follows from the Cauchy-Schwarz inequality, and by adding and subtracting $\langle\delta \mb{x}_\star/\left\Vert\mb{x}_\star\right\Vert_2,\mb{h} \rangle$. Applying the Cauchy-Schwarz inequality on the first two terms of second line, yields the third line. The fourth line then follows by observing the orthogonal decompositions of $\mb{a}_0$ and $\mb{h}$ with respect to $\mb{x}_\star$.
The obtained inequality implies the inclusion
\begin{align}
\mc A\subset\mc A_{\delta} & \defeq\left\{ \mb h\,:\,\sqrt{1-\delta^2}\left\lVert \mb h\right\rVert _{2}+\langle\frac{\delta\mb x_{\star}}{\left\lVert \mb x_{\star}\right\rVert _{2}}-\mb g,\mb h\rangle\ge0,\mr{\ for\ all\ }\mb g\in\partial\Omega\left(\mb x_{\star}\right)\right\} \,.\label{eq:A_delta}
\end{align}
 It suffices to show that $\mc A_{\delta}$ does not contain any point with (relatively)
large $\ell_{2}$ norm consistent with the constraints as mentioned above. 

As it becomes clear in the sequel, it is critical that $\mc{A}_\delta$ excludes a sufficiently large subset of $\mbb{R}^N$. Therefore, we may require $\delta$ to be bounded away from zero. For example, in unstructured phase retrieval, where $f_m(\mb{x})=(\mb{a}_m^\T\mb{x})^2$ and $\Omega(\mb{x}) = 0$, the set $\mc{A}_\delta$ should not contain $-\mb{x}_\star$. Therefore, in this case we must have $\sqrt{1-\delta^2} < \delta$ or equivalently $\delta >1/\sqrt{2}$. Throughout, we implicitly assume that such required lower bounds on $\delta$ hold. Furthermore, instead of \eqref{eq:A_delta} we could have used the tighter approximation
\begin{align*}
\mc{A} & \subseteq \left\{\mb{h}\,:\,\langle\frac{\mb x_{\star}}{\left\lVert \mb x_{\star}\right\rVert _{2}},\mb{h}\rangle<\delta \left\lVert \mb h\right\rVert _{2}\ \text{and}\ \sqrt{1-\delta^{2}}\left\lVert \mb h_{\perp}\right\rVert _{2}+\langle\frac{\delta\mb x_{\star}}{\left\lVert \mb x_{\star}\right\rVert _{2}}-\mb{g},\mb h\rangle \ge 0\ \text{for all}\ \mb g\in\partial\Omega\left(\mb x_{\star}\right)\right\}\\
& \hspace{3ex}\bigcup \left\{\mb{h}\,:\, \langle\frac{\mb x_{\star}}{\left\lVert \mb x_{\star}\right\rVert _{2}},\mb{h}\rangle\ge\delta \left\lVert \mb h\right\rVert _{2}\ \text{and}\  \left\lVert \mb h\right\rVert _{2}-\langle\mb{g},\mb{h}\rangle \ge 0\ \text{for all}\ \mb g\in\partial\Omega\left(\mb x_{\star}\right)\right\}\,,
\end{align*}
 where $\mb{h}_\perp$ denotes part of $\mb{h}$ that is orthogonal to $\mb{x}_\star$. While using this approximation improves the dependence of our result on $\delta$, we prefer \eqref{eq:A_delta}
merely for simpler notation and derivations.

  Our main theorem below provides a sample complexity
for establishing the desired sufficient condition and thus accuracy
of \eqref{eq:anchored-regression} in terms of a Rademacher complexity and a probability bound for  $\nabla f(\mb{x}_\star)$ being in certain half-spaces. Let us pause here to describe these two quantities first. 

For a set $\mc H\subset\mbb R^{N}$ define the Rademacher complexity\footnote{Unlike conventional definition of Rademacher complexities, we use a normalization by square root of the number of samples.} with respect to $\nabla f_m(\mb{x}_\star)$'s as
\begin{align}
\mfk C_{M}\left(\mc H\right)\defeq & \ \E\sup_{\mb h\in\mc H}\,\frac{1}{\sqrt{M}}\sum_{m=1}^{M}\epsilon_{m}\langle\nabla f_{m}\left(\mb x_{\star}\right),\frac{\mb h}{\left\lVert \mb h\right\rVert _{2}}\rangle\,,\label{eq:complexity}
\end{align}
where $\epsilon_{1},\epsilon_{2},\dotsc,\epsilon_{M}$ are i.i.d. Rademacher random variables independent of everything else. If $\mc{H}$ is non-negative homogeneous (e.g., it is a convex cone), then $\mfk{C}_M(\mc{H})$ is a measure of ``wideness'' of $\mc{H}$ near the origin; the notion of wideness here  depends on the law of $\nabla f(\mb{x}_\star)$. For example, in the case that $\nabla f(\mb{x}_\star)$ is a standard Gaussian random vector, the vector $M^{-1/2}\sum_{m=1}^M \epsilon_m \nabla f_m (\mb{x}_\star)\rangle$ is also distributed like a standard Gaussian vector and $\mfk{C}_M(\mc{H})$ reduces to the \emph{Gaussian width} of $\mc{H}$ which is defined as
\begin{align*}
	\gamma(\mc{H}) &\defeq \E \sup_{\mb{h}\in\mc{H}}\,\langle \mb{g},\frac{\mb{h}}{\left\lVert\mb{h}\right\rVert_2}\rangle,
	\qquad \mb{g}\sim\mr{Normal}(\mb{0},\mb{I})\,.
\end{align*}
The reduction of the Rademacher complexity $\mfk{C}_M(\mc{H})$ to the Gaussian width $\gamma(\mc{H})$ is also possible for $\nabla f(\mb{x}_\star)$ that is not Gaussian but has a sufficiently regular law: for any fixed $\mb{h}$, assuming bounded moments of sufficiently high order the random quantity $M^{-1/2}\sum_{m=1}^M\epsilon_m\langle \nabla f_m(\mb{x}_\star),\mb{h}\rangle$ can be approximated by a Gaussian using Berry-Ess\'{e}en theorem; we only need to make this approximation uniform over the entire $\mc{H}$. While the Gaussian width of $\mc{H}$ can provide sharp approximations for $\mfk{C}_M(\mc{H})$, it is generally a difficult quantity to compute. In Section \ref{ssec:RadCom} we use different techniques to bound the Rademacher complexity in the special cases of unstructured and sparse regression problems.

The second quantity that affects the sample complexity of our method is a probability lower bound defined for $\mc{H}$ and  a positive parameter $\tau$ as
\begin{align}
p_{\tau}\left(\mc H\right) & \defeq \inf_{\mb h\in\mc H}\,\P\left(\langle\nabla f\left(\mb x_{\star}\right),\mb h\rangle\ge\tau\left\lVert \mb h\right\rVert _{2}\right)\,.\label{eq:small-ball-probability}
\end{align}
Intuitively, $p_\tau (\mc{H})$ measures how ``well-spread''  the random vector $\nabla f(\mb{x}_\star)$ is in the space. A smaller value of $p_\tau(\mc{H})$ indicates that realizations of $\nabla f(\mb{x}_\star)$ are often confined to some half-space whose normal vector belongs to $\mc{H}$. 

We now state our main theorem in terms of the quantities $\mfk{C}_M(\mc{A}_\delta)$ and  $p_\tau(\mc{A}_\delta)$.  In Sections~\ref{ssec:ProbBound} and \ref{ssec:RadCom} below, we show how these quantities can be bounded in a way that shows their dependence on the probability law for $\nabla f(\mb{x}_\star)$ more clearly.

\begin{thm}
\label{thm:main-theorem}Let $\mc A_{\delta}$ be defined as in \eqref{eq:A_delta} for which $\mfk{C}_M(\mc{A}_\delta)$ and $p_\tau(\mc{A}_\delta)$ can be determined using \eqref{eq:complexity}  and \eqref{eq:small-ball-probability}, respectively. For any  $t>0$, if 
\begin{align}
M & \ge 4\left(\frac{2\mfk C_{M}\left(\mc A_{\delta}\right)+t \tau}{\tau p_{\tau}\left(\mc A_{\delta}\right)}\right)^{2}\,,\label{eq:sample-complexity}
\end{align}
 then with probability $\ge 1-\exp(-2t^2)$ any solution
$\widehat{\mb x}$ of \eqref{eq:anchored-regression} obeys 
\begin{align*}
\left\lVert \widehat{\mb x}-\mb x_{\star}\right\rVert _{2} & \le\frac{2}{\tau p_{\tau}\left(\mc A_{\delta}\right)} \left(\frac{1}{M}\sum_{m=1}^{M}\left|\xi_{m}\right|+\varepsilon\right)\,.
\end{align*}
\end{thm}
The factor $4$ in \eqref{eq:sample-complexity} can be made arbitrarily close to $1$ at the cost of increasing the constant factor in the error bound.  Proof of Theorem \ref{thm:main-theorem}, provided below in Section \ref{sec:proof}, is based on
the idea of \emph{small-ball method} introduced in \citep{Koltchinskii-Bounding-2015,Mendelson-Learning-2014}
and further developed in \citep{Mendelson-Learning-2014b,Lecue-Regularization_I-2016,Lecue-Regularization_II_2016}.

\subsection{Simplifying $p_\tau(\mc{A}_\delta)$ and the dependence on $\tau$}
\label{ssec:ProbBound}

The sample complexity in the statement of Theorem \ref{thm:main-theorem} critically depends on the parameters $\tau$ and $p_\tau(\mc{A}_\delta)$.  Approximating $p_\tau(\mc{A}_\delta)$ is challenging in general and may require detailed calculations even for specific cases (see, for example, \cite[Lemmas 3 and 5]{Bahmani-Phase-2016} where this is done for the phase retrieval problem $f_m(\mb{x}) = |\langle\mb{x},\mb{a}_m\rangle|^2$ with Gaussian vectors $\mb{a}_m$).  In this section, we show how $p_\tau(\mc{A}_\delta)$ and $\tau$ can be approximated by expressions that capture the interplay between the probability distribution on the $f_m$ and the geometry of the set $\mc{A}_\delta$.  These may still be very difficult to calculate, but they will allow us to state more interpretable corollaries to our main theorem,

We start by observing that
\begin{align*}
	\P\left(\langle\nabla f(\mb{x}_\star),\mb{h}\rangle \ge \tau \left\lVert\mb{h}\right\rVert_2 \right) 
	& =\P\left(\left(\langle\nabla f(\mb{x}_\star),\mb{h}\rangle\right)_+ \ge \tau \left\lVert\mb{h}\right\rVert_2 \right)\,,
\end{align*}
and then define 
\begin{equation}
	\label{eq:taudefs}
	\tau(\mb{h}) \defeq \frac{\E \left(\langle\nabla f(\mb{x}_\star),\mb{h}\rangle\right)_+}{\left\lVert\mb{h}\right\rVert_2},
	\quad
	\tau(\mc{A}_\delta) \defeq \inf_{\mb{h}\in\mc{A}_\delta\backslash\{\mb{0}\}}~\tau(\mb{h}).
\end{equation}

Clearly, both of these quantities are non-negative. However, to ensure that $p_\tau(\mc{A}_\delta)>0$ for some $\tau>0$, it is necessary that $\tau(\mc{A}_\delta)$ is strictly positive.  The only way this does \textit{not} occur is when there is some $\mb{h}\in \mc{A}_\delta\backslash\{\mb{0}\}$ for which $\langle \nabla f(\mb{x}_\star),\mb{h}\rangle\le 0$ almost surely.  That is, with probability $1$, $\nabla f(\mb{x}_\star)$ lies in a half-space whose normal vector is in $\mc{A}_\delta$. Qualitatively, if the distribution for $\nabla f(\mb{x}_\star)$ is well-spread, then this situation does not occur, and $\tau(\mc{A}_\delta)$ is positive.

By the Paley-Zygmund inequality \cite{PaleyZygmund-Analytic-1932}\cite[Corollary 3.3.2]{delaPenaGine-Decoupling-1999}, 
\[
	\P\left(\left(\langle\nabla f(\mb{x}_\star),\mb{h}\rangle\right)_+ \ge \frac{1}{2}\,\tau(\mb{h})\|\mb{h}\|_2 \right)
	~\geq~
	\frac{\tau^2(\mb{h})\|\mb{h}\|_2^2}{4\E[(\langle\nabla f(\mb{x}_\star),\mb{h}\rangle)_+^2]}.
\]
Since $(\langle\nabla f(\mb{x}_\star),\mb{h}\rangle)_+^2\leq |\nabla f(\mb{x}_\star)^\T\mb{h}|^2=\mb{h}^\T\nabla f(\mb{x}_\star)\nabla f(\mb{x}_\star)^\T\mb{h}$, we have for $\tau=\tau(\mc{A}_\delta)/2$,
\begin{equation}
	\label{eq:ptaulower}
	p_\tau(\mc{A}_\delta) ~\geq~ \frac{\tau^2(\mc{A}_\delta)}{4\varsigma^2(\mc{A}_\delta)},
\end{equation}
where
\begin{align}
	\label{eq:DAdelta}
	\varsigma^2(\mc{A}_\delta) &\defeq 
	\sup_{\mb{h}\in\mc{A}_\delta\backslash\{\mb{0}\}} \frac{\mb{h}^\T\mb{\varSigma}_\star\mb{h}}{\|\mb{h}\|_2^2}\,,
	&
	\text{with}\ \mb{\varSigma}_\star &= \E\left[\nabla f(\mb{x}_\star)\nabla f(\mb{x}_\star)^\T\right]\,.
\end{align}
This gives us the following corollary to Theorem~\ref{thm:main-theorem}:
\begin{cor}\label{cor:main}
	Let $\mc{A}_\delta$ be defined as in \eqref{eq:A_delta}, and $\mfk{C}_M(\mc{A}_\delta),\tau(\mc{A}_\delta),\varsigma^2(\mc{A}_\delta)$ as in \eqref{eq:complexity}, \eqref{eq:taudefs}, and \eqref{eq:DAdelta}.  For any $t>0$, if
	\begin{equation}
		\label{eq:Mcorlower}
		M ~\geq~ \frac{64\,\varsigma^4(\mc{A}_\delta)}{\tau^4(\mc{A}_\delta)}\left(\frac{4\mfk{C}_M(\mc{A}_\delta)}{\tau(\mc{A}_\delta)} + t \right)^2,
	\end{equation}
	then with probability $\ge 1-\exp(-2t^2)$, any solution $\widehat{\mb{x}}$ of \eqref{eq:anchored-regression} obeys
	\begin{equation}
		\label{eq:corerror}
		\|\widehat{\mb{x}}-\mb{x}_\star\|_2 ~\leq~ \left(\frac{16\,\varsigma^2(\mc{A}_\delta)}{\tau^3(\mc{A}_\delta)}\right)\left(\frac{1}{M}\sum_{m=1}^M|\xi_m|+\varepsilon\right).
	\end{equation}
\end{cor}

A simple dimensional analysis can help us understand how the bound above scale qualitatively.  The entries in the gradient vector $\nabla f(\mb{x})$ represent quantities in ``units of $f$'' divided by ``units of $\mb{x}$''.  The constants $\varsigma(\mc{A}_\delta),\tau(\mc{A}_\delta),$ and $\mfk{C}_M(\mc{A}_\delta)$ all have the same units as $\nabla f(\mb{x})$.   Thus the lower bound on the number of equations $M$ in \eqref{eq:Mcorlower} is dimensionless, as we expect.  As the noise variables $\xi_m$ are also in ``units of $f$'', both sides of the inequality in \eqref{eq:corerror} are in ``units of $\mb{x}$''.

\subsection{Bounding the Rademacher complexity $\protect\mfk C_{M}\left(\protect\mc A_{\delta}\right)$}
\label{ssec:RadCom}

To make the result of Theorem \ref{thm:main-theorem} more explicit
we consider the special cases of unstructured regression (i.e., $\mb x_{\star}$
is arbitrary and $\Omega\left(\mb x\right)=0$) and sparse regression
(i.e., $\mb x_{\star}$ is sparse and $\Omega\left(\mb x\right)=\lambda\left\lVert \mb x\right\rVert _{1}$
for some $\lambda>0$). The specific choice of $\Omega\left(\mb x\right)$
helps to simplify the Rademacher complexity $\mfk C_{M}\left(\mc A_{\delta}\right)$.
Of course, $\mfk C_{M}\left(\mc A_{\delta}\right)$ would depend on
the law of $f_{1},f_{2},\dotsc,f_{M}$ as well.

\subsubsection{Unstructured regression}

In the first example, we approximate $\mfk C_{M}\left(\mc A_{\delta}\right)$
when no regularization is applied, i.e., $\Omega\left(\mb x\right)=0$.
In this case, \eqref{eq:A_delta} reduces to 
\begin{align*}
\mc A_{\delta} & =\left\{ \mb h\,:\,\sqrt{1-\delta^2}\left\lVert \mb h\right\rVert _{2}+\langle\frac{\delta\mb x_{\star}}{\left\lVert \mb x_{\star}\right\rVert _{2}},\mb h\rangle\ge0\right\} \,,
\end{align*}
which we approximate by the entire space (i.e., $\mc A_{\delta}\subseteq\mbb R^{N})$.
Because $\mfk C_{M}\left(\mc{\cdot}\right)$ is monotonic, we obtain
\begin{align*}
\mfk C_{M}\left(\mc A_{\delta}\right) & \le\mfk C_{M}\left(\mbb R^{N}\right)\\
 & =\ \E\sup_{\mb h}\,\frac{1}{\sqrt{M}}\sum_{m=1}^{M}\epsilon_{m}\langle\nabla f_{m}\left(\mb x_{\star}\right),\frac{\mb h}{\left\lVert \mb h\right\rVert _{2}}\rangle\\
 & =\ \E\left\lVert \frac{1}{\sqrt{M}}\sum_{m=1}^{M}\epsilon_{m}\nabla f_{m}\left(\mb x_{\star}\right)\right\rVert _{2}\,.
\end{align*}
Applying the Cauchy-Schwarz for the square-root function which is
concave, we deduce that 
\begin{align*}
\mfk C_{M}\left(\mc A_{\delta}\right) & \le\sqrt{\E\left\lVert \frac{1}{\sqrt{M}}\sum_{m=1}^{M}\epsilon_{m}\nabla f_{m}\left(\mb x_{\star}\right)\right\rVert _{2}^{2}}\\
 & =\sqrt{\frac{1}{M}\E\sum_{m=1}^{M}\left\lVert \nabla f_{m}\left(\mb x_{\star}\right)\right\rVert _{2}^{2}}\\
 & =\sqrt{\E\left\lVert \nabla f\left(\mb x_{\star}\right)\right\rVert _{2}^{2}}\,,
\end{align*}
where the second and third lines respectively hold because $\epsilon_{m}$s
are independent zero-mean random variables and $f_{m}$s are i.i.d.
copies of $f$.
With this bound and using the fact that $\varsigma^2(\mc{A}_\delta)\leq\|\mb{\varSigma}_\star\|$, Corollary \ref{cor:main} can be specialized to the following with all of the relevant definitions unchanged.
\begin{cor}
\label{cor:unstruct}
For any $t>0$, if 
\begin{align}
	\label{eq:Munstruct}
	M & \ge \frac{64\|\mb{\varSigma}_\star\|^2}{\tau^4(\mc{A}_\delta)}\left(\frac{4\sqrt{\operatorname{tr}(\mb{\varSigma}_\star)}}{\tau(\mc{A}_\delta)} + t\right)^{2}\,,
\end{align}
then with probability $\ge1-\exp(-2t^{2})$ any solution $\widehat{\mb x}$
of \eqref{eq:anchored-regression} obeys 
\begin{align*}
\left\lVert \widehat{\mb x}-\mb x_{\star}\right\rVert _{2} & \le \frac{16\left\Vert\mb{\varSigma}_\star\right\Vert}{\tau^3\left(\mc{A}_\delta\right)}\left(\frac{1}{M}\sum_{m=1}^{M}\left|\xi_{m}\right|+\varepsilon\right)\,.
\end{align*}
\end{cor}

Since $\operatorname{tr}(\mb{\varSigma_\star})\leq N\|\mb{\varSigma_\star}\|$, it suffices to take
\[
	M ~\geq~ \frac{64\|\mb{\varSigma}_\star\|^2}{\tau^4(\mc{A}_\delta)}\left(4\sqrt{N}\,\frac{\sqrt{\|\mb{\varSigma}_\star\|}}{\tau(\mc{A}_\delta)} + t\right)^{2}.
\]
When $\sqrt{\|\mb{\varSigma}_\star\|}/\tau(\mc{A}_\delta)$ is on the order of a constant, this means we can take $M\gtrsim N$.  Provided we have an anchor vector that obeys \eqref{eq:anchor}, we can robustly recover a vector $\mb{x}_\star$ of length $N$ through observations of slightly more than $N$ convex functions of $\mb{x}_\star$. 

To see the above result in more explicit examples, we consider some special cases where with $\mb{a}\sim\mr{Normal}(\mb{0},\mb{I})$, we have $f(\mb{x}) = \phi(\mb{a}^\T \mb{x})$ for linear regression ($\phi(z) = z$), phase retrieval over the real numbers ($\phi(z) = z^2$), and regression with a \textit{rectified linear unit} (ReLU) ($\phi(z)=(z)_+$). In the latter special case the assumed differentiability of $f$ is violated, but the arguments still hold if, with abuse of notation, we define $\phi'(t)=\bbone(t\ge 0)$ and consequently $\nabla f(\mb{x}) = \bbone(\mb{a}^\T\mb{x}\ge 0)\mb{a}$ which is a true subgradient. 

\begin{exmp}\label{exmp:unstructured-regression}
In view of the discussion above, we find appropriate approximations for  $\left\lVert \mb{\varSigma}_\star \right\rVert$ and $\tau(\mc{A_\delta})$ which are the crucial quantities in the derived sample complexity and error bound,  in the context of linear regression, phase retrieval, and ReLU regression, as described above. 
We have 
\begin{align*}
\E\left((\langle\nabla f(\mb{x}_\star),\mb{h}\rangle)_+\right) &= \E\left((\phi'(\mb{a}^\T \mb{x}_\star) \mb{a}^\T\mb{h})_+\right)\\
																						 &=\frac{1}{2} \E\left(\phi'(\mb{a}^\T \mb{x}_\star) \mb{a}^\T\mb{h}\right) + \frac{1}{2}\E\left(\left|\phi'(\mb{a}^\T \mb{x}_\star) \mb{a}^\T\mb{h}\right|\right)
\end{align*}
In the case of linear regression we have $\phi(z) = z$ which yields
 \begin{align*}
\E\left((\langle\nabla f(\mb{x}_\star),\mb{h}\rangle)_+\right) &= \frac{1}{2} \E\left(\mb{a}^\T\mb{h}\right) + \frac{1}{2}\E\left(\left|\mb{a}^\T\mb{h}\right|\right)\\
																						&= \frac{1}{\sqrt{2\pi}}\left\Vert\mb{h}\right\Vert_2\,, 	
\end{align*}
which implies that $\tau(\mc{A}_\delta) = 1/\sqrt{2\pi}$.
Since $\mb{\varSigma}_\star  = \E\left(\mb{a}\mb{a}^\T\right)=\mb{I}$, we deduce that $\sqrt{\left\Vert\mb{\varSigma}_\star\right\rVert}/\tau(\mc{A_\delta})= \sqrt{2\pi}$.

In the case of phase retrieval where $\phi(z)=z^2$, using some Gaussian integral calculations \citep[Corollary 3.1]{Li-Gaussian-2009} we obtain 
\begin{align}
\E\left((\langle\nabla f(\mb{x}_\star),\mb{h}\rangle)_+\right) &=2\E\left((\mb{a}^\T \mb{x}_\star) (\mb{a}^\T\mb{h})\right) + \E\left(\left|(\mb{a}^\T \mb{x}_\star) (\mb{a}^\T\mb{h})\right|\right)\nonumber\\
																						 &=2\left\Vert\mb{x}_\star\right\Vert_2 \left\Vert\mb{h}\right\Vert_2 \left(r(\mb{h}) + \frac{2}{\pi}\left(\sqrt{1-r^2(\mb{h})}+r(\mb{h}) \arcsin\left(r(\mb{h})\right)\right)  \right)\,,\label{eq:PR-tau}
\end{align}
where $r(\mb{h}) = \tfrac{\mb{x}_\star^\T \mb{h}}{\left\Vert\mb{x}_\star\right\Vert_2\left\Vert \mb{h}\right\Vert_2}$. By definition, for every $\mb{h}\!\in\!\mc{A}_\delta$ we have $r(\mb{h})\!\ge\! \max\lbrace-1,\!-\tfrac{\sqrt{1-\delta^2}}{\delta}\rbrace$. Because $z \mapsto z +\frac{2}{\pi}(\sqrt{1-z^2} + z\arcsin z)$ is increasing in $z$, we deduce that for a sufficiently large $\delta\le 1$ there exists a constant $c>0$ such that 
\[\E\left((\langle\nabla f(\mb{x}_\star),\mb{h}\rangle)_+\right) \ge c \left\Vert \mb{x}_\star\right\Vert_2\left\Vert\mb{h}\right\Vert_2\,,\]
for every $\mb{h}\in\mc{A}_\delta$, which implies that $\tau(\mc{A}_\delta) \ge c\left\Vert\mb{x}_\star\right\Vert_2$. Furthermore, we have 
\[\mb{\varSigma}_\star =2 \E\left((\mb{a}^\T\mb{x}_\star)^2\mb{a}\mb{a}^\T\right)=4\mb{x}_\star\mb{x}_\star^\T + 2\left\Vert\mb{x}_\star\right\Vert_2^2\mb{I}\,,\]
and thus $\left\Vert\mb{\varSigma}_\star\right\Vert = 6\left\Vert\mb{x}_\star\right\Vert_2^2 $. Therefore, $\sqrt{\left\Vert\mb{\varSigma}_\star\right\rVert}/\tau(\mc{A_\delta})= \sqrt{6}/c$.

Finally, in the case of ReLU regression where $\phi(z)=(z)_+$, we have 
\begin{align*}
\E\left((\langle\nabla f(\mb{x}_\star),\mb{h}\rangle)_+\right) &= \E\left(\bbone(\mb{a}^\T\mb{x}_\star \ge 0 )(\mb{a}^\T\mb{h})_+\right)\\
																						 &= \frac{1}{4}\E\left((1+\sgn(\mb{a}^\T\mb{x}_\star))(\mb{a}^\T\mb{h}+\left|\mb{a}^\T\mb{h}\right|)\right)\\
																						 &= \frac{1}{4}\E\left(\sgn(\mb{a}^\T\mb{x}_\star)(\mb{a}^\T\mb{h})\right)+\frac{1}{4}\E\left(\left|\mb{a}^\T\mb{h}\right|\right)\,,
\end{align*}
where we used the fact that $\mb{a}$ has a symmetric probability density to obtain the third equation. Let $\mb{h}_\perp$ denote the projection of $\mb{h}$ on the orthogonal complement of $\mr{span}\left(\lbrace\mb{x}_\star\rbrace\right)$. With this notation we can write 
\begin{align*}
\E\left((\langle\nabla f(\mb{x}_\star),\mb{h}\rangle)_+\right) &= \frac{1}{4}\E\left(\sgn(\mb{a}^\T\mb{x}_\star)(\mb{a}^\T(\mb{h}_\perp+\mb{h}-\mb{h}_\perp)\right)+\frac{1}{4}\E\left(\left|\mb{a}^\T\mb{h}\right|\right)\,.
\end{align*}
Because $\mb{a}$ has a standard Gaussian density, $\mb{a}^\T\mb{x}_\star$ and $\mb{a}^\T\mb{h}_\perp$ are independent. Therefore, with $r(\mb{h})$ as in the case of phase retrieval above, and using the fact that $\mb{h}-\mb{h}_\perp = ((\mb{x}_\star^\T \mb{h})/\left\Vert \mb{x}_\star\right\Vert_2^2)\,\mb{x}_\star$, we deduce that
\begin{align*}
\E\left((\langle\nabla f(\mb{x}_\star),\mb{h}\rangle)_+\right) &= \frac{1}{4}\E\left(\sgn(\mb{a}^\T\mb{x}_\star)(\mb{a}^\T(\mb{h}-\mb{h}_\perp)\right)+\frac{1}{4}\E\left(\left|\mb{a}^\T\mb{h}\right|\right)\\
																						 &= \frac{\mb{x}_\star^\T\mb{h}}{4\left\Vert\mb{x}_\star\right\Vert_2^2}\E\left(\left|\mb{a}^\T\mb{x}_\star \right|\right)+\frac{1}{\sqrt{8\pi}}\left\Vert\mb{h}\right\Vert_2\\																						 
																						 &=\frac{1+r(\mb{h})}{\sqrt{8\pi}}\left\Vert\mb{h}\right\Vert_2\,.
\end{align*}
Since $r(\mb{h})\ge\max \lbrace -1, -\sqrt{1-\delta^2}/\delta\rbrace$ for every $\mb{h}\in\mc{A}_\delta$, we have $\tau(\mc{A}_\delta)\!\ge\! \tfrac{1}{\sqrt{8\pi}}\left(1\!-\!\sqrt{1\!-\!\delta^2}/\delta\right)_+$. Thus, for a sufficiently large $\delta\le 1$ we can have $\tau(\mc{A}_\delta)\ge c$ for some constant $c>0$. Because $\mb{\varSigma}_\star = \E \left(\bbone\left(\mb{a}^\T\mb{x}_\star \ge 0\right)\mb{a}\mb{a}^\T \right)=\frac{1}{2}\mb{I}$, we obtain $\sqrt{\left\Vert\mb{\varSigma}_\star\right\Vert}/\tau(\mc{A}_\delta) \le 1/(c\sqrt{2})$. Following the ideas described in Section \ref{sec:recipes} below, we can show that $\mb{a}_0 = -\tfrac{1}{N}\sum_{i=1}^N y_i \mb{a}_i$ can be used as the anchor for ReLU regression and the overall sample complexity is $M=O(N)$. This result can be compared with the result of \cite{Soltanolkotabi-Learning-2017} that considers the non-convex (projected) gradient descent on the non-linear least squares formulation of the ReLU regression with regularization. Under the same Gaussian observation model, \cite{Soltanolkotabi-Learning-2017} establishes a sample complexity for the gradient descent initialized at the origin to converge at a linear rate. In the unregularized case above, this sample complexity is $M=O(N)$ that we also obtained. The analysis of \cite{Soltanolkotabi-Learning-2017} relies on the notion of \textit{Gaussian width} to measure the sample complexity; the Gaussian distribution of the vectors $\mb{a}_m$ is critical in this analysis. Our analysis applies more broadly as we only require to appropriately bound the constants $\tau(\mc{A}_\delta)$ and $\varsigma(\mc{A}_\delta)$.
\end{exmp}

\subsubsection{Sparse regression}

Next, we approximate $\mfk C_{M}\left(\mc A_{\delta}\right)$ for
the case of $\ell_{1}$-regularized anchored regression, i.e., $\Omega\left(\mb x\right)=\lambda\left\lVert \mb x\right\rVert _{1}$
for some $\lambda>0$. Let $\mc S_{\star}$ denote the support set
of $\mb x_{\star}$ (i.e., $\mc S_{\star}=\left\{ i\in\left[1,N\right]\,:\,x_{\star i}\ne0\right\} $)
and $s=\left|\mc S_{\star}\right|$. The subdifferential of the $\Omega\left(\cdot\right)$
at $\mb x_{\star}$ can be expressed as 
\begin{align*}
\partial\Omega\left(\mb x_{\star}\right) & =\left\{ \mb g\,:\,\left\lVert \mb g\right\rVert _{\infty}\le\lambda\text{ and }{\mb g\vert}_{\mc S_{\star}}=\lambda\sgn\left({\mb x_{\star}\vert}_{\mc S_{\star}}\right)\right\} \,.
\end{align*}
Therefore, we have 
\begin{align}
\mc A_{\delta} & =\left\{ \mb h\,:\,\sqrt{1-\delta^2}\left\lVert \mb h\right\rVert _{2}-\lambda\left\lVert {\mb h\vert}_{\mc S_{\star}^{\mr c}}\right\rVert _{1}+\langle\frac{\delta{\mb x_{\star}\vert}_{\mc S_{\star}}}{\left\lVert \mb x_{\star}\right\rVert _{2}}-\lambda\sgn\left({\mb x_{\star}\vert}_{\mc S_{\star}}\right),{\mb h\vert}_{\mc S_{\star}}\rangle\ge0\right\} \,. \label{eq:Adelta-sparse}
\end{align}
 It follows from the Cauchy-Schwarz inequality that
\begin{align*}
\langle\frac{\delta{\mb x_{\star}\vert}_{\mc S_{\star}}}{\left\lVert \mb x_{\star}\right\rVert _{2}}-\lambda\sgn\left({\mb x_{\star}\vert}_{\mc S_{\star}}\right),{\mb h\vert}_{\mc S_{\star}}\rangle & \le\left\lVert \frac{\delta{\mb x_{\star}\vert}_{\mc S_{\star}}}{\left\lVert \mb x_{\star}\right\rVert _{2}}-\lambda\sgn\left({\mb x_{\star}\vert}_{\mc S_{\star}}\right)\right\rVert _{2}\left\lVert {\mb h\vert}_{\mc S_{\star}}\right\rVert _{2}\\
 & =\left\lVert \frac{\delta\mb x_{\star}}{\left\lVert \mb x_{\star}\right\rVert _{2}}-\lambda\sgn\left(\mb x_{\star}\right)\right\rVert _{2}\,\left\lVert {\mb h\vert}_{\mc S_{\star}}\right\rVert _{2}\,,
\end{align*}
 thereby
\begin{align*}
\mc A_{\delta} & \subseteq\mc A'_{\delta}\defeq\left\lbrace\mb h\,:\,\sqrt{1-\delta^2}\left\lVert \mb h\right\rVert _{2}-\lambda\left\lVert {\mb h\vert}_{\mc S_{\star}^{\mr c}}\right\rVert _{1}+\left\lVert \frac{\delta\mb x_{\star}}{\left\lVert \mb x_{\star}\right\rVert _{2}}-\lambda\sgn\left(\mb x_{\star}\right)\right\rVert _{2}\left\lVert {\mb h\vert}_{\mc S_{\star}}\right\rVert _{2}\ge0\right\rbrace\,.
\end{align*}
 For any non-zero $\mb h$ and any $\mb z\in\mbb R^{N}$, it follows
from H\"{o}lder's inequality that 
\begin{align}
\langle\mb z,\frac{\mb h}{\left\lVert \mb h\right\rVert _{2}}\rangle & \le\left\lVert {\mb z\vert}_{\mc S_{\star}}\right\rVert _{2}\frac{\left\lVert {\mb h\vert}_{\mc S_{\star}}\right\rVert _{2}}{\left\lVert \mb h\right\rVert _{2}}+\left\lVert {\mb z\vert}_{\mc S_{\star}^{\mr c}}\right\rVert _{\infty}\frac{\left\lVert {\mb h\vert}_{\mc S_{\star}^{\mr c}}\right\rVert _{1}}{\left\lVert \mb h\right\rVert _{2}}\,.\label{eq:head-tail-bound}
\end{align}
 Therefore, recalling the definition \eqref{eq:complexity}, we can
write
\begin{align*}
\mfk C_{M}\left(\mc A_{\delta}\right) & \le\mfk C_{M}\left(\mc A'_{\delta}\right)\\
 & =\ \E\sup_{\mb h\in\mc A'_{\delta}}\,\frac{1}{\sqrt{M}}\sum_{m=1}^{M}\epsilon_{m}\langle\nabla f_{m}\left(\mb x_{\star}\right),\frac{\mb h}{\left\lVert \mb h\right\rVert _{2}}\rangle\\
 & \le\ \E\,\sup_{\mb h\in\mc A'_{\delta}}\,\left\lVert \frac{1}{\sqrt{M}}\sum_{m=1}^{M}\epsilon_{m}{\nabla f_{m}\left(\mb x_{\star}\right)\vert}_{\mc S_{\star}}\right\rVert _{2}\frac{\left\lVert {\mb h\vert}_{\mc S_{\star}}\right\rVert _{2}}{\left\lVert \mb h\right\rVert _{2}}\\
 & \hspace{3ex}+\E\,\sup_{\mb h\in\mc A'_{\delta}}\,\left\lVert \frac{1}{\sqrt{M}}\sum_{m=1}^{M}\epsilon_{m}{\nabla f_{m}\left(\mb x_{\star}\right)\vert}_{\mc S_{\star}^{\mr c}}\right\rVert _{\infty}\frac{\left\lVert {\mb h\vert}_{\mc S_{\star}^{\mr c}}\right\rVert _{1}}{\left\lVert \mb h\right\rVert _{2}}
\end{align*}
where the first inequality follows from the fact that $\mc A_{\delta}\subseteq\mc A'_{\delta}$
and the second inequality is an application of \eqref{eq:head-tail-bound}
for $\mb z=\frac{1}{\sqrt{M}}\sum_{m=1}^{M}\epsilon_{m}\nabla f_{m}\left(\mb x_{\star}\right)$
followed by a triangle inequality. We always have $\frac{\left\lVert {\mb h\vert}_{\mc S_{\star}}\right\rVert _{2}}{\left\lVert \mb h\right\rVert _{2}}\le1$
and the definition of $\mc A'_{\delta}$ implies that 
\begin{align*}
\frac{\left\lVert {\mb h\vert}_{\mc S_{\star}^{\mr c}}\right\rVert _{1}}{\left\lVert \mb h\right\rVert _{2}} & \le\frac{\sqrt{1-\delta^2}}{\lambda}+\frac{1}{\lambda}\,\left\lVert \frac{\delta\mb x_{\star}}{\left\lVert \mb x_{\star}\right\rVert _{2}}-\lambda\sgn\left(\mb x_{\star}\right)\right\rVert _{2}\frac{\left\lVert {\mb h\vert}_{\mc S_{\star}}\right\rVert _{2}}{\left\lVert \mb h\right\rVert _{2}}\\
 & \le\frac{\sqrt{1-\delta^2}}{\lambda}+\frac{1}{\lambda}\,\left\lVert \frac{\delta\mb x_{\star}}{\left\lVert \mb x_{\star}\right\rVert _{2}}-\lambda\sgn\left(\mb x_{\star}\right)\right\rVert _{2}\,.
\end{align*}
 Hence, we deduce that 
\begin{align*}
\mfk C_{M}\left(\mc A_{\delta}\right) & \le\E\,\left\lVert \frac{1}{\sqrt{M}}\sum_{m=1}^{M}\epsilon_{m}{\nabla f_{m}\left(\mb x_{\star}\right)\vert}_{\mc S_{\star}}\right\rVert _{2}\\
 & \hspace{3ex}+\left(\frac{\sqrt{1-\delta^2}}{\lambda}+\frac{1}{\lambda}\,\left\lVert \frac{\delta\mb x_{\star}}{\left\lVert \mb x_{\star}\right\rVert _{2}}-\lambda\sgn\left(\mb x_{\star}\right)\right\rVert _{2}\right)\E\,\left\lVert \frac{1}{\sqrt{M}}\sum_{m=1}^{M}\epsilon_{m}{\nabla f_{m}\left(\mb x_{\star}\right)\vert}_{\mc S_{\star}^{\mr c}}\right\rVert _{\infty}\,.
\end{align*}
 Next, we bound the terms with the expectation. For the first term,
the Cauchy-Schwarz inequality and the fact that $\epsilon_{m}$s are
independent and zero-mean guarantees that 
\begin{align*}
\E\,\left\lVert \frac{1}{\sqrt{M}}\sum_{m=1}^{M}\epsilon_{m}{\nabla f_{m}\left(\mb x_{\star}\right)\vert}_{\mc S_{\star}}\right\rVert _{2} & \le\sqrt{\E\,\left\lVert \frac{1}{\sqrt{M}}\sum_{m=1}^{M}\epsilon_{m}{\nabla f_{m}\left(\mb x_{\star}\right)\vert}_{\mc S_{\star}}\right\rVert _{2}^{2}}\\
 & =\sqrt{\frac{1}{M}\sum_{m=1}^{M}\E\,\left\lVert {\nabla f_{m}\left(\mb x_{\star}\right)\vert}_{\mc S_{\star}}\right\rVert _{2}^{2}}\\
 & =\sqrt{\E\,\left\lVert {\nabla f\left(\mb x_{\star}\right)\vert}_{\mc S_{\star}}\right\rVert _{2}^{2}}\,.
\end{align*}
For the second term, a similar application of the Cauchy-Schwarz inequality,
followed by the Nemirovski's inequality (see \citep[Lemma 5.2.2]{Nemirovski-Topics-2000}
and \citep[ Theorem 2.2]{Dumbgen-Nemirovski-2010}) with the minor
assumption that $N-s\ge3$, yields
\begin{align*}
\E\,\left\lVert \frac{1}{\sqrt{M}}\sum_{m=1}^{M}\epsilon_{m}{\nabla f_{m}\left(\mb x_{\star}\right)\vert}_{\mc S_{\star}^{\mr c}}\right\rVert _{\infty} & \le\sqrt{\E\,\left\lVert \frac{1}{\sqrt{M}}\sum_{m=1}^{M}\epsilon_{m}{\nabla f_{m}\left(\mb x_{\star}\right)\vert}_{\mc S_{\star}^{\mr c}}\right\rVert _{\infty}^{2}}\\
 & \le\sqrt{2e\log\left(N-s\right)\,\cdot\frac{1}{M}\sum_{m=1}^{M}\E\,\left\lVert {\nabla f_{m}\left(\mb x_{\star}\right)\vert}_{\mc S_{\star}^{\mr c}}\right\rVert _{\infty}^{2}}\\
 & \le\sqrt{2e\log N\,\E\,\left\lVert {\nabla f\left(\mb x_{\star}\right)\vert}_{\mc S_{\star}^{\mr c}}\right\rVert _{\infty}^{2}}\,.
\end{align*}
Recalling the definitions  \eqref{eq:taudefs} and  \eqref{eq:DAdelta}, we can specialize Corollary \ref{cor:main} as follows.
\begin{cor}
\label{cor:sparse-regression} Let \[C_{\delta,\lambda}^{\star}=\sqrt{2e}\left(\frac{\sqrt{1-\delta^2}}{\lambda}+\frac{1}{\lambda}\,\left\lVert \frac{\delta\mb x_{\star}}{\left\lVert \mb x_{\star}\right\rVert _{2}}-\lambda\sgn\left(\mb x_{\star}\right)\right\rVert _{2}\right)\,.\]
For any $t>0$, if 
\begin{align*}
M & \ge \frac{64\,\varsigma^4(\mc{A}_\delta)}{\tau^4(\mc{A}_\delta)}\left(\frac{4\sqrt{\E\,\left\lVert {\nabla f\left(\mb x_{\star}\right)\vert}_{\mc S_{\star}}\right\rVert _{2}^{2}}+4C_{\delta,\lambda}^{\star}\sqrt{\log\left(N\right)\,\E\,\left\lVert {\nabla f\left(\mb x_{\star}\right)\vert}_{\mc S_{\star}^{\mr c}}\right\rVert _{\infty}^{2}}}{\tau(\mc{A}_\delta)}+t\right)^{2}\,,
\end{align*}
then with probability $\ge1-\exp(-2t^{2})$ any solution $\widehat{\mb x}$
of \eqref{eq:anchored-regression} obeys 
\begin{align*}
\left\lVert \widehat{\mb x}-\mb x_{\star}\right\rVert _{2} &\le  \frac{16\,\varsigma^2(\mc{A}_\delta)}{\tau^3(\mc{A}_{\delta})}\left(\frac{1}{M}\sum_{m=1}^{M}\left|\xi_{m}\right|+\varepsilon\right)\,.
\end{align*}
\end{cor}
Let us consider the particular case of regression with linear predictors to make the result of Corollary \ref{cor:sparse-regression} more explicit.
\begin{exmp}[Sparse regression with linear predictor]
Suppose that $f\left(\mb x\right)=\phi(\mb a^{\T}\mb x)$
for $\mb a\!\sim\!\mr{Normal}\left(\mb 0,\mb I\right)$ and a convex $\phi(\cdot)$ that has derivatives of sufficiently high order. Moreover, in the calculations below we implicitly assume that the expectations that involve $\phi(\cdot)$ or its derivatives exist. Clearly, we can write $\nabla f\left(\mb x_{\star}\right)=\phi'(\mb{a}^\T\mb{x}_\star)\mb a$ and obtain
\begin{align*}
\E\left\lVert {\nabla f\left(\mb x_{\star}\right)\vert}_{\mc S_{\star}}\right\rVert _{2}^{2} & =  \E \left({\phi'}^2(\mb{a}^\T\mb{x}_\star) \left\lVert {\mb{a}\vert}_{\mc{S}_\star}\right\rVert_2^2\right)\\&=\E \left({\phi'}^2( \left\lVert\mb{x}_\star\right\rVert_2 g)(g^2 + \left|\mc S_{\star}\right| - 1) \right) , \qquad g\sim \mr{Normal}(0,1)\\
& =s\,\E\left({\phi'}^2(\left\lVert \mb{x}_\star \right\rVert_2 g) \right) \\
& \phantom{=}  + 2\E\left( {\phi^{''}}^2(\left\lVert \mb{x}_\star \right\rVert_2 g) + \phi'(\left\lVert \mb{x}_\star \right\rVert_2 g) \phi^{'''}(\left\lVert \mb{x}_\star \right\rVert_2 g) \right)\left\lVert\mb{x}_\star \right\rVert_2^2\,.
\end{align*}
where the second equation follows from independence of projections of ${\mb{a}\vert}_{\mc{S}_\star}$ onto subspaces parallel and orthogonal to $\mb{x}_\star$. The third equation follows by applying the Stein's lemma \cite{Stein-Bound-1972} twice.
Furthermore, we have 
\begin{align*}
\E\left\lVert {\nabla f\left(\mb x_{\star}\right)\vert}_{\mc S_{\star}^{\mr c}}\right\rVert _{\infty}^{2} &=  \E\left({\phi'}^2(\mb{a}^\T\mb{x}_\star)\max_{i\in \mc{S}_\star^\mr{c}} a_i^2\right)\\
& = \E\left({\phi'}^2(\mb{a}^\T\mb{x}_\star)\right)\E\left(\max_{i\in \mc{S}_\star^\mr{c}} a_i^2\right)\\
&\le \E\left({\phi'}^2(\left\lVert\mb{x}_\star\right\rVert_2 g)\right)\E\left(\sum_{i\in\mc{S}_\star^\mr{c}} a_i^{2q}\right)^{\frac{1}{q}}, \qquad g \sim\mr{Normal}(0,1)\\
& \le \E\left({\phi'}^2(\left\lVert\mb{x}_\star\right\rVert_2 g)\right) \left((N-s)\E\left(g^{2q}\right)\right)^\frac{1}{q}\,,
\end{align*}
where the second line holds by independence of $\mb{a}^\T\mb{x}_\star$ and $a_i$ for $i\in \mc{S}_\star^\mr{c}$, the third line holds for any positive integer $q$, and the fourth line follows from concavity of $t\mapsto t^{1/q}$ and the Jensen's inequality. Hence, for $q = O(\log N)$ we obtain
\[\E\left\lVert {\nabla f\left(\mb x_{\star}\right)\vert}_{\mc S_{\star}^{\mr c}}\right\rVert _{\infty}^{2} \lesssim \E\left({\phi'}^2(\left\lVert\mb{x}_\star\right\rVert_2 g)\right) \log N\,. \]
Treating the terms that only depend on $\left\lVert\mb{x}_\star\right\rVert_2$ as constants we can summarize the derived bounds as
\begin{align*}
\E\left\lVert {\nabla f\left(\mb x_{\star}\right)\vert}_{\mc S_{\star}}\right\rVert _{2}^{2} & \lesssim s\,, & \text{and} & &
\E\left\lVert {\nabla f\left(\mb x_{\star}\right)\vert}_{\mc S_{\star}^{\mr c}}\right\rVert _{\infty}^{2} 
&\lesssim \log N\,.
\end{align*} 
If we choose the regularization parameter as $\lambda = O(1/\sqrt{s})$ for a suitable constant factor that depends on $\delta$, then we obtain
\begin{align*}
C_{\delta,\lambda}^{\star} & \le c_{\delta}\sqrt{s}\,,
\end{align*}
 where $c_{\delta}$ is a constant decreasing in $\delta$. Therefore, we can conclude that having
\begin{align*}
M &\gtrsim \frac{\varsigma^4(\mc{A}_\delta)}{\tau^4(\mc{A}_\delta)}\left(\frac{\sqrt{s}+c_{\delta}\sqrt{s}\log\left(N\right)}{\tau\left(\mc A_{\delta}\right)}+t\right)^{2}
\end{align*}
is sufficient to meet the requirements in Corollary \ref{cor:sparse-regression}.
Assuming that $\varsigma(\mc{A}_\delta)$ and $\tau(\mc{A}_\delta)$ are constant terms, the result above suggests a sample complexity of $O\left(s\,\log^{2}N\right)$
which is optimal up to the dependence on $\log N$. A more refined bound can be obtained using a tighter bound for $\E\,\left\lVert \frac{1}{\sqrt{M}}\sum_{m=1}^{M}\epsilon_{m}{\nabla f_{m}\left(\mb x_{\star}\right)\vert}_{\mc S_{\star}^{\mr c}}\right\rVert _{\infty}$ instead of invoking the Nemirovski's inequality in the proof of Corollary \ref{cor:sparse-regression}.\looseness=-1

To complement the discussion above, we verify that $\varsigma(\mc{A}_\delta)$ and $\tau(\mc{A}_\delta)$ are constant terms  in the cases of \textit{compressed sensing} (i.e. $\phi(z) = z$) and \textit{sparse phase retrieval} (i.e., $\phi(z) = z^2$). This is easy for compressed sensing, because the derivations for linear regression in Example \ref{exmp:unstructured-regression} still apply. For sparse phase retrieval, similar to Example \ref{exmp:unstructured-regression}, $\varsigma(\mc{A}_\delta)$ can be bounded as 
\[\varsigma^2(\mc{A}_\delta) \le \left\Vert\mb{\varSigma}_\star\right\Vert =6\left\Vert\mb{x}_\star\right\Vert_2^2\,.\]
To bound $\tau(\mc{A}_\delta)$ we can use \eqref{eq:PR-tau} with $r(\mb{h})=\mb{x}_\star^\T\mb{h}/(\left\Vert\mb{x}_\star\right\Vert_2\left\Vert\mb{h}\right\Vert_2)$ as before. Recalling the definition \eqref{eq:Adelta-sparse} of $\mc{A}_{\delta}$, it is straightforward to show that $r(\mb{h})\! \ge\! \max\lbrace -1, -\tfrac{\sqrt{1-\delta^2}}{\delta}-\tfrac{\lambda(2\sqrt{s}-1)}{\delta}\rbrace$. Therefore, following a similar argument as in Example \ref{exmp:unstructured-regression}, for a sufficiently large $\delta\le 1$ and a suitable choice of the parameter $\lambda=O(1/\sqrt{s})$ we can show that $\tau(\mc{A}_\delta)$ is bounded below by a positive absolute constant.
\end{exmp}

\section{\label{sec:recipes}Recipes for creating anchors}
While it may be assumed that the anchor vector is provided by an oracle,
it is more realistic to have a data-driven method to construct the
anchor vector. This requires us to impose new assumptions, albeit
implicitly, on the class of functions $\mc F$ where the random samples
$f_{1},f_{2},\dotsc,f_{M}$ are drawn from. A natural assumption is
that for a certain sample loss function $\ell:\mbb R\times\mbb R\to\mbb R$
the corresponding \emph{risk} $R\left(\mb x\right)\defeq\E\ell\left(f\left(\mb x\right),f(\mb x_{\star})\right)$
``encodes'' the information about $\mb x_{\star}$ in its derivatives
at some reference point $\mb{x}_0$, i.e., $\nabla^{k}R\left(\mb{x}_0\right)$ for $k=1,2,3,\dotsc$. For simplicity, henceforth we take the origin as the reference point, i.e., $\mb{x}_0 = \mb{0}$.
This assumption is reasonable, for instance, if $R\left(\mb x\right)$
has a Taylor series approximation around the origin. Note that we
are implicitly assuming that the random functions drawn from the class
$\mc F$ as well as the sample loss $\ell\left(\cdot,\cdot\right)$
are $k$-times differentiable. Furthermore, to allow the derivative
operator and expectation to commute, the required regularity conditions
for the law of $f$ are assumed.
\begin{defn}[Spiked Derivatives]
We say that the risk $R\left(\mb x\right)=\E\ell\left(f\left(\mb x\right),f\left(\mb x_{\star}\right)\right)$
has a \emph{spiked $k\textup{th}$-order derivative} if $\mb x_{\star}$
is a \emph{simple principal eigenvector} of $-\nabla^{k}R\left(\mb 0\right)\ne\mb 0$.
Namely,
\begin{align}
\frac{\mb x_{\star}}{\left\lVert \mb x_{\star}\right\rVert _{2}}\in\argmax_{\mb u\in\mbb S^{N-1}}\  & \langle-\nabla^{k}R\left(\mb 0\right),\mb u^{\otimes k}\rangle\,,\label{eq:derviative-conditions}
\end{align}
 and every other maximizer is parallel to $\mb x_{\star}$.
\end{defn}
For computational considerations, we focus only on risks with \emph{spiked
gradient} (i.e., 1st-order derivative) or \emph{spiked Hessian} (i.e.,
2nd-order derivative).\emph{ }In the case of spiked gradients the
condition \eqref{eq:derviative-conditions} reduces to $-\nabla R\left(\mb 0\right)$
being perfectly aligned with $\mb x_{\star}$, that is, 
\begin{align}
\frac{\mb x_{\star}}{\left\lVert \mb x_{\star}\right\rVert _{2}} & =-\frac{\nabla R\left(\mb 0\right)}{\left\lVert \nabla R\left(\mb 0\right)\right\rVert _{2}}\,.\label{eq:gradient}
\end{align}
 Similarly, for spiked Hessians, the condition \eqref{eq:derviative-conditions}
reduces to $\mb x_{\star}$ being a \emph{simple} \emph{principal
eigenvector} of $-\nabla^{2}R\left(\mb 0\right)$. Specifically, we
have 
\begin{equation}
\begin{aligned}\frac{\mb x_{\star}}{\left\lVert \mb x_{\star}\right\rVert _{2}} & \in\argmax_{\mb u\in\mbb S^{N-1}}\ -\mb u^{\T}\nabla^{2}R\left(\mb 0\right)\mb u\\
\gamma_{\star} & \defeq\lambda_{1}\left(-\nabla^{2}R\left(\mb 0\right)\right)-\lambda_{2}\left(-\nabla^{2}R\left(\mb 0\right)\right)>0
\end{aligned}
\,,\label{eq:Hessian}
\end{equation}
where $\lambda_{i}\left(\cdot\right)$ denotes the $i$th largest
(multiplicity inclusive) eigenvalue of its argument. For example,
with the squared error $\ell\left(s,t\right)=\frac{1}{2}\left(s-t\right)^{2}$
as the sample loss function, if the risk $R\left(\mb x\right)=\E\,\frac{1}{2}\left(f\left(\mb x\right)-f\left(\mb x_{\star}\right)\right)^{2}$
has a spiked gradient then
\begin{align*}
\frac{\mb x_{\star}}{\left\lVert \mb x_{\star}\right\rVert _{2}} & =\frac{1}{\left\lVert \E\left(\left(f(\mb x_{\star})-f(\mb 0)\right)\nabla f\left(\mb 0\right)\right)\right\rVert _{2}}\E\left(\left(f(\mb x_{\star})-f(\mb 0)\right)\nabla f\left(\mb 0\right)\right)\,.
\end{align*}
 and if it has spiked Hessian then 
\begin{align*}
\frac{\mb x_{\star}}{\left\lVert \mb x_{\star}\right\rVert _{2}} & \in\argmax_{\mb u\in\mbb S^{N-1}}\ \mb u^{\T}\E\left(f(\mb x_{\star})\nabla^{2}f\left(\mb 0\right)-\nabla f\left(\mb 0\right)\nabla^{\T}f\left(\mb 0\right)\right)\mb u\,,
\end{align*}
and the Hessian has a positive spectral gap.

Let us again consider the problem of regression with linear predictors as a concrete example.
\begin{exmp}[Spiked derivatives in regression with linear predictor]\label{exmp:spiked-derivatives} Let $f(\mb{x})=\phi(\mb{a}^\T\mb{x})$ for $\mb{a}\sim\mr{Normal}(\mb{0},\mb{I})$ and a twice-differentiable convex function  $\phi(\cdot)$. Without loss of generality we may assume that $\phi(0) = 0$. Suppose that we observe $M$ i.i.d. samples of $f(\mb{x}_\star)$, i.e., $y_m = f_m(\mb{x}_\star) = \phi(\mb{a}_m^\T\mb{x}_\star)$ for $m=1,2,\dotsc,M$ with $\mb{a}_m$s being i.i.d. copies of $\mb{a}$. The risk with respect to the squared loss $\ell(u,v) = \tfrac{1}{2}(u-v)^2$ is 
\[R(\mb{x}) = \frac{1}{2}\E \left(\phi(\mb{a}^\T\mb{x}) - \phi(\mb{a}^\T\mb{x}_\star)\right)^2\,.\] Therefore, we can calculate $\nabla R(\mb 0)$ as 
\begin{align*}
\nabla R(\mb{0}) & = \E\left(\phi'(0)\left(\phi(0) - \phi(\mb{a}^\T\mb{x}_\star) \right)\mb{a}\right)\\
& = -\phi'(0) \E\left(\phi(\mb{a}^\T\mb{x}_\star)\mb{a}\right)\\
& = -\phi'(0)\E\left(\phi'(\left\lVert\mb{x}_\star\right\rVert_2g)\right)\mb{x}_\star\,, & g\sim\mr{Normal}(0,1)\,, 
\end{align*}
where the third line follows from Stein's lemma. Similarly, we can write
\begin{align*}
\nabla^2 R(\mb{0}) &= \E\left(\left({\phi'}^2(0) - \phi''(0)\phi(\mb{a}^\T\mb{x}_\star)\right)\mb{a}\mb{a}^\T\right) \\
& = \left({\phi'}^2(0) - \phi''(0)\E\left(\phi(\mb{a}^\T\mb{x}_\star)\right) \right)\mb{I} - \phi''(0)\E\left(\phi''(\mb{a}^\T\mb{x}_\star)\right)\mb{x}_\star\mb{x}_\star^\T\\
& = \left({\phi'}^2(0) - \phi''(0)\E\left(\phi(\left\lVert\mb{x}_\star\right\rVert_2 g)\right) \right)\mb{I} - \phi''(0)\E\left(\phi''(\left\lVert\mb{x}_\star\right\rVert_2 g)\right)\mb{x}_\star\mb{x}_\star^\T\,.
\end{align*}
If $\phi'(0)\ne 0$ then $R(\cdot)$ clearly has a spiked gradient. If $\phi'(0)=0$, however, $\nabla R(\mb{0})=\mb{0}$ and we need to inspect the second derivative. In that case, if $\phi''(0)\ne 0$ then we have
 \[\nabla^2 R(\mb{0}) = - \phi''(0)\E\left(\phi''(\left\lVert\mb{x}_\star\right\rVert_2 g)\right)\mb{x}_\star\mb{x}_\star^\T - \phi''(0)\E\left(\phi(\left\lVert\mb{x}_\star\right\rVert_2 g)\right)\mb{I}\,.\]
 Because $\phi(\cdot)$ is convex, $\phi''(\cdot)$ is non-negative and the matrix above is clearly a spiked Hessian with spectral gap $\gamma_\star = \phi''(0)\left\lVert\mb{x}_\star\right\rVert_2^2 \E\left(\phi''(\left\lVert\mb{x}_\star\right\rVert_2 g) \right)$. By a simple change of variable the derivations above can be extended to the case where the Gaussian random vector $\mb{a}$ is still zero-mean, but has an arbitrary covariance matrix.

 Let us look at the special cases considered in Example \ref{exmp:unstructured-regression}, namely,  linear regression, phase retrieval, and ReLU regression. In linear regression $\phi(z)=z$, thus we have $\nabla R(\mb{0})=-\mb{x}_\star$. Similarly, for ReLU regression, where $\phi(z)=(z)_+$, defining $\phi'(z)=\bbone(z\ge0)$ as in Example \ref{exmp:unstructured-regression}, we obtain
 \begin{align*}
 \nabla R(\mb{0}) &= -\E\left(\bbone(g\ge 0)\right)\mb{x}_\star =-\frac{1}{2}\E\left(1 + \sgn(g)\right)\mb{x}_\star = -\frac{1}{2}\mb{x}_\star
 \end{align*}
 Therefore, we have spiked gradients in both of the considered linear and ReLU regression models. In the case of phase retrieval we have $\phi(z)=z^2$ and it is clear that $\nabla R(\mb{0})=\mb{0}$. However, by straightforward calculations we obtain
 \begin{align*}
 \nabla^2 R(\mb{0}) &= -2\E\left|\mb a^{\T}\mb x_{\star}\right|^{2}\mb a\mb a^{\T}=-4\mb{x}_\star \mb{x}_\star^\T - 2\left\Vert\mb{x}_\star \right\Vert^2_2\mb{I}\,,
 \end{align*}
 which means that we have a spiked Hessian with the spectral gap $\gamma_\star = 4\Vert\mb{x}_\star\Vert_2^2$. This property is leveraged in the ``spectal initialization'' for non-convex phase retrieval methods (see, e.g., \cite{Candes_Phase_2014}).
 \end{exmp}
 
Of course, in practice we do not have access to $R\left(\mb x\right)$.
With finite number of observations, however, the empirical risk 
\begin{align*}
R_{M}\left(\mb x\right) & =\frac{1}{M}\sum_{m=1}^{M}\ell\left(f_{m}\left(\mb x\right),y_{m}\right)\,,
\end{align*}
 may provide a sufficiently good approximation for $R\left(\mb x\right)$.
Therefore, with sufficient number of observations, $\nabla R_{M}\left(\mb 0\right)$,
as an approximation to $\nabla R\left(\mb 0\right)$, or $\nabla^{2}R_{M}\left(\mb 0\right)$,
as an approximation to $\nabla^{2}R\left(\mb 0\right)$, can be used
to find an anchor vector under the spiked gradient or spiked Hessian
conditions, respectively. 

In case of the spiked gradient, it suffices
to have sufficiently large number of samples (i.e., $M$) such that
\begin{align*}
\mb a_{0} & =-\frac{\nabla R_{M}\left(\mb 0\right)}{\left\lVert \nabla R_{M}\left(\mb 0\right)\right\rVert }_{2}\,,
\end{align*}
obeys anchor vectors' required property \eqref{eq:anchor}. Suppose that the law of $f$ is such that $\nabla R_M(\mb{0})$ concentrates around $\nabla R(\mb{0})$, namely 
\[ \left\lVert \nabla R_M(\mb{0}) - \nabla R(\mb{0})\right\rVert_2\le \epsilon \left\lVert \nabla R(\mb{0}) \right\rVert_2\,, \]
with high probability for  a small $\epsilon \in (0,1)$. If $R(\cdot)$ has a spiked gradient, then we obtain
\begin{align*}
\langle \mb{a}_0,\frac{\mb{x}_\star}{\left\lVert \mb{x}_\star\right\rVert_2}\rangle & =  \langle \frac{\nabla R_{M}\left(\mb 0\right)}{\left\lVert \nabla R_{M}\left(\mb 0\right)\right\rVert }_{2}, \frac{\nabla R\left(\mb 0\right)}{\left\lVert \nabla R\left(\mb 0\right)\right\rVert }_{2}\rangle \\
 & \ge \frac{(1-\epsilon^2){\left\lVert \nabla R\left(\mb 0\right)\right\rVert}_2^2+ {\left\lVert \nabla R_M\left(\mb 0\right)\right\rVert}_2^2}{2{\left\lVert \nabla R_M\left(\mb 0\right)\right\rVert }_{2}{\left\lVert \nabla R\left(\mb 0\right)\right\rVert}_2}\\
 & \ge \sqrt{1-\epsilon^2}\,,
\end{align*}
where the third line follows from the AM-GM inequality. Therefore, $\mb{a}_0$ would satisfy \eqref{eq:anchor} with $\delta = \sqrt{1-\epsilon^2}$.

In the case that $R\left(\cdot\right)$ has a spiked Hessian, using a variant
of the Davis\textendash Kahan's theorem \citep[Corollary 3]{Yu-DavisKahan-2014},
we can show that for 
\begin{align}
\mb a_{0} & \in\argmax_{\mb u\in\mbb S^{N-1}}\ -\mb u^{\T}\nabla^{2}R_{M}\left(\mb 0\right)\mb u\,,\label{eq:empirical-Hessian}
\end{align}
we have 
\begin{align*}
\left\lVert \mb a_{0}\mb a_{0}^{\T}-\frac{1}{\left\lVert \mb x_{\star}\right\rVert _{2}^{2}}\mb x_{\star}\mb x_{\star}^{\T}\right\rVert  & \le\frac{2\left\lVert \nabla^{2}R_{M}\left(\mb 0\right)-\nabla^{2}R\left(\mb 0\right)\right\rVert }{\gamma_{\star}}\,,
\end{align*}
where $\gamma_{\star}$ is the spectral gap defined in \eqref{eq:Hessian}.
Therefore, if $\left\lVert \nabla^{2}R_{M}\left(\mb 0\right)-\nabla^{2}R\left(\mb 0\right)\right\rVert $
is sufficiently small relative to the spectral gap $\gamma_{\star}$, the inequality above implies  $\mb a_{0}$ (or $-\mb a_{0}$) can
obey the required property \eqref{eq:anchor} for some $\delta>0$.
Depending on the law of $f$, we can bound $\left\lVert \nabla^{2}R_{M}\left(\mb 0\right)-\nabla^{2}R\left(\mb 0\right)\right\rVert $
using matrix concentration inequalities as $\nabla^{2}R_{M}\left(\mb 0\right)$
can be written as a sum of independent random matrices.

As a concrete example,  consider phase retrieval for a (real-valued) target $\mb x_{\star}$ using noiseless measurements obtained through
i.i.d. copies of $\mb a\sim\mr{Normal}\left(\mb 0,\mb I\right)$ as
the measurement vectors. As mentioned in Example \ref{exmp:spiked-derivatives} above, the risk
$R\left(\mb x\right)=\E\tfrac{1}{2}\left(\left|\mb a^{\T}\mb x\right|^{2}-\left|\mb a^{\T}\mb x_{\star}\right|^{2}\right)^{2}$
has a spiked Hessian because $-\nabla^{2}R\left(\mb 0\right)=4\mb x_{\star}\mb x_{\star}^{\T} + 2\left\lVert \mb x_{\star}\right\rVert _{2}^{2}\mb I$
whose spectral gap is $\gamma_{\star}=4\left\lVert \mb x_{\star}\right\rVert _{2}^{2}$.  Indeed, the leading eigenvector of $\nabla^2R_M(\mb 0)$ can be shown to be sufficiently correlated with $\mb{x}_\star$ for $M\gtrsim N\log N$ \cite{Candes_Phase_2014}, and is used as the anchor vector in \cite{Bahmani-Phase-2016,bahmani17fl}.

\subsection{Anchors for structured regression} 
How does construction of the anchor change if we need to estimate
a structured ground truth (e.g., a sparse vector)? In this scenario,
the situation is more complicated compared to what described above
mainly because computationally efficient methods may not achieve optimal sample complexity. 

Ideally, we would have a method for constructing the anchor vector from a number of observations that does not dominate the sample complexity of the estimator.  This may be possible if we impose an explicit structural constraint
$\mb u\in\mc S$ with $\mc S$ denoting a prescribed set of structured
vectors; namely we would have 
\begin{align}
\mb a_{0} & \in\argmax_{\mb u\in\mbb S^{N-1},\mb u\in\mc S}\ \langle-\nabla^{k}R_{M}\left(\mb 0\right),\mb u^{\otimes k}\rangle\,.\label{eq:structured-anchor}
\end{align}
Regardless of whether this construction can achieve \eqref{eq:anchor} at an appropriate sample complexity, there is no guarantee that solving \eqref{eq:structured-anchor} is computationally tractable.
For some of the problems discussed above, there are relaxations
of \eqref{eq:structured-anchor} that produce an anchor, but at the cost 
of increasing $M$ well beyond the required sample complexity
\eqref{eq:sample-complexity} of the main estimation procedure. 


We can turn again to the phase retrieval problem as a specific example; this time considering the target $\mb{x}_\star$ to be \textit{sparse}.
Ignoring the computational restrictions, if $\mc{S}$ is the set of (sufficiently) sparse vectors, a brute force estimator generally solves \eqref{eq:structured-anchor} for $k=2$ by searching over all possible sparse support sets. This estimator can produce
the desired $\mb a_{0}$ with (near) optimal sample complexity as we only need small submatrices of $-\nabla^{2}R_{M}\left(\mb 0\right)$ to concentrate around their expected values. For sparse phase retrieval, the desired concentration occurs when $M$ (i.e., the number of samples) grows (nearly) linearly with the sparsity of $\mb{x}_\star$  (i.e., $\left\lVert \mb x_{\star}\right\rVert _{0}$). However, computationally tractable relaxations of \eqref{eq:structured-anchor}, which rely on diagonal thresholding or mixed nuclear-norm  $\ell_1$-norm regularization, require $M$ to grow quadratically with $\left\lVert \mb x_{\star}\right\rVert _{0}$ to guarantee accuracy. This situation is similar to the case of \emph{sparse principal component
analysis} (SPCA) \citep{Zou-SPCA-2006,Johnstone-Consistency-2009}, where the goal is to estimate the sparse principal component of a (covariance) matrix from an empirical (covariance) matrix. It is widely believed that computationally efficient estimators cannot achieve the optimal sample complexity in SPCA \cite{Berthet_Complexity_2013}. Suboptimality of mixed nuclear-norm $\ell_1$-norm regularization is also shown in estimation of simultaneously sparse and low-rank matrices \cite{Oymak-Simultaneously-2015}.

\section{\label{sec:proof}Proof of the main result}

In this section we provide a proof of Theorem \ref{thm:main-theorem}.
The argument is based on the small-ball method introduced in \citep{Koltchinskii-Bounding-2015,Mendelson-Learning-2014}
with minor modifications. Our derivations mostly follow the exposition
of this method in \citep{Tropp-Convex-2015}.
\begin{proof}[Proof of Theorem \ref{thm:main-theorem}]
For $t\ge0$, let $\psi_{t}\left(s\right)\defeq\left(s\right)_{+}-\left(s-t\right)_{+}$
which is a \emph{contraction} (i.e., $\left|\psi_{\tau}\left(s_{2}\right)\!-\!\psi_{\tau}\left(s_{1}\right)\right|\le\left|s_{2}-s_{1}\right|$
for all $s_{1},s_{2}\in\mbb R$). We can bound $R_{M}^{\,+}\left(\mb x_{\star}+\mb h\right)$
from below as
\begin{align}
R_{M}^{\,+}\left(\mb x_{\star}+\mb h\right) & =\frac{1}{M}\sum_{m=1}^{M}\left(f_{m}\left(\mb x_{\star}+\mb h\right)-y_{m}\right)_{+}\nonumber \\
 & =\frac{1}{M}\sum_{m=1}^{M}\left(f_{m}\left(\mb x_{\star}+\mb h\right)-f_{m}\left(\mb x_{\star}\right)-\xi_{m}\right)_{+}\nonumber \\
 & \ge\frac{1}{M}\sum_{m=1}^{M}\left(\left\langle \nabla f_{m}\left(\mb x_{\star}\right),\mb h\right\rangle -\xi_{m}\right)_{+}\nonumber \\
 & \ge\frac{1}{M}\sum_{m=1}^{M}\left(\left\langle \nabla f_{m}\left(\mb x_{\star}\right),\mb h\right\rangle \right)_{+}-\frac{1}{M}\sum_{m=1}^{M}\left(\xi_{m}\right)_{+}\nonumber \\
 & \ge\frac{1}{M}\sum_{m=1}^{M}\psi_{\tau\left\lVert \mb h\right\rVert _{2}}\left(\langle\nabla f_{m}\left(\mb x_{\star}\right),\mb h\rangle\right)-\frac{1}{M}\sum_{m=1}^{M}\left(\xi_{m}\right)_{+}\,,\label{eq:main-inequality}
\end{align}
where the inequalities hold, respectively, because $f_{m}\left(\cdot\right)$
is convex, $\left(\cdot\right)_{+}$ is subadditive, and $\left(s\right)_{+}\ge\psi_{t}\left(s\right)$.
Furthermore, using the fact that $t\bbone\left(s\ge t\right)\le\psi_{t}\left(s\right)$,
we have
\begin{align*}
\tau\left\lVert \mb h\right\rVert _{2}\P\left(\langle\nabla f\left(\mb x_{\star}\right),\mb h\rangle\ge\tau\left\lVert \mb h\right\rVert _{2}\right) & =\tau\left\lVert \mb h\right\rVert _{2}\,\E\left(\bbone\left(\langle\nabla f\left(\mb x_{\star}\right),\mb h\rangle\ge\tau\left\lVert \mb h\right\rVert _{2}\right)\right)\\
 & \le\E\left(\psi_{\tau\left\lVert \mb h\right\rVert _{2}}\left(\langle\nabla f\left(\mb x_{\star}\right),\mb h\rangle\right)\right)
\end{align*}
 Therefore, by adding and subtracting the sides of the inequality
above in the right-hand side of \eqref{eq:main-inequality} and some
rearrangement we obtain
\begin{equation}
\begin{aligned}R_{M}^{\,+}\left(\mb x_{\star}+\mb h\right) & \ge\tau\left\lVert \mb h\right\rVert _{2}\P\left(\langle\nabla f\left(\mb x_{\star}\right),\mb h\rangle\ge\tau\left\lVert \mb h\right\rVert _{2}\right)-\frac{1}{M}\sum_{m=1}^{M}\left(\xi_{m}\right)_{+}\\
 & \hspace{3ex}+\frac{1}{M}\sum_{m=1}^{M}\psi_{\tau\left\lVert \mb h\right\rVert _{2}}\left(\langle\nabla f_{m}\left(\mb x_{\star}\right),\mb h\rangle\right)-\E\left(\psi_{\tau\left\lVert \mb h\right\rVert _{2}}\left(\langle\nabla f\left(\mb x_{\star}\right),\mb h\rangle\right)\right)\,.
\end{aligned}
\label{eq:penultimate-inequality}
\end{equation}
We only need to establish a uniform lower bound over $\mb h\in\mc A_{\delta}$
for the expression in the second line of \eqref{eq:penultimate-inequality}.
It is easy to verify that for every $\alpha,t\ge0$ and $s\in\mbb R$
the identity $\psi_{\alpha t}\left(s\right)=t\psi_{\alpha}\left(\frac{s}{t}\right)$
holds.\footnote{Because $\psi_{\alpha}\left(\cdot\right)$ is bounded, we can treat
$t=0$ as $t\to0$ to avoid the issue of division by zero.} Thus, we write 
\[
\begin{aligned}\frac{1}{M}\sum_{m=1}^{M}\psi_{\tau\left\lVert \mb h\right\rVert _{2}}\left(\langle\nabla f_{m}\left(\mb x_{\star}\right),\mb h\rangle\right)-\E\left(\psi_{\tau\left\lVert \mb h\right\rVert _{2}}\left(\langle\nabla f\left(\mb x_{\star}\right),\mb h\rangle\right)\right)\\
=-\frac{1}{M}\left\lVert \mb h\right\rVert _{2}\sum_{m=1}^{M}\E\left(\psi_{\tau}\left(\langle\nabla f\left(\mb x_{\star}\right),\frac{\mb h}{\left\lVert \mb h\right\rVert _{2}}\rangle\right)\right)-\psi_{\tau}\left(\langle\nabla f_{m}\left(\mb x_{\star}\right),\frac{\mb h}{\left\lVert \mb h\right\rVert _{2}}\rangle\right)\,,
\end{aligned}
\]
and we only need to find an upper bound for 
\begin{align*}
F_{\mc A_{\delta}}\left(f_{1},f_{2},\dotsc,f_{M}\right) & \defeq\sup_{\mb h\in\mc A_{\delta}}\,\frac{1}{M}\sum_{m=1}^{M}\E\left(\psi_{\tau}\left(\langle\nabla f\left(\mb x_{\star}\right),\frac{\mb h}{\left\lVert \mb h\right\rVert _{2}}\rangle\right)\right)-\psi_{\tau}\left(\langle\nabla f_{m}\left(\mb x_{\star}\right),\frac{\mb h}{\left\lVert \mb h\right\rVert _{2}}\rangle\right)\,.
\end{align*}
Since $\psi_{\tau}\left(\cdot\right)$ is bounded by $0$ and $\tau$, if we replace any one of the vectors $\nabla f_m(\mb{x}_\star)$ with some other vector while keeping the rest intact, the value of $F(\cdot)$ wont change by more than $\tau/M$. Therefore, a standard application of the \emph{bounded difference inequality}
\citep{McDiarmid-Bounded-1989} to $F_{\mc A_{\delta}}\left(\cdot\right)$
shows that for any $t>0$ we have
\begin{align}
F_{\mc A_{\delta}}\left(f_{1},f_{2},\dotsc,f_{M}\right) & \le\E F_{\mc A_{\delta}}\left(f_{1},f_{2},\dotsc,f_{M}\right)+\frac{t\tau}{\sqrt{M}}\,,\label{eq:BDI}
\end{align}
 with probability $\ge1-\exp(-2t^2)$. 
 
It remains only to upper bound $\E F_{\mc A_{\delta}}\!\left(f_{1},f_{2},\dotsc,f_{M}\right)$.
Writing the inner expectation in $\E F_{\mc A_{\delta}}\!\left(f_{1},f_{2},\dotsc,f_{M}\right)$
with respect to random functions $\widetilde{f}_{m}$ that are i.i.d. copies of $f$,
independent of everything else, we obtain
\begin{align*}
\E F_{\mc A_{\delta}}\left(f_{1},f_{2},\dotsc,f_{M}\right) & =\E\sup_{\mb h\in\mc A_{\delta}}\,\frac{1}{M}\sum_{m=1}^{M}\E\left(\psi_{\tau}\left(\langle\nabla f\left(\mb x_{\star}\right),\frac{\mb h}{\left\lVert \mb h\right\rVert _{2}}\rangle\right)\right)-\psi_{\tau}\left(\langle\nabla f_{m}\left(\mb x_{\star}\right),\frac{\mb h}{\left\lVert \mb h\right\rVert _{2}}\rangle\right)\\
 & \le\E\sup_{\mb h\in\mc A_{\delta}}\,\frac{1}{M}\sum_{m=1}^{M}\psi_{\tau}\left(\langle\nabla\widetilde{f}_{m}\left(\mb x_{\star}\right),\frac{\mb h}{\left\lVert \mb h\right\rVert _{2}}\rangle\right)-\psi_{\tau}\left(\langle\nabla f_{m}\left(\mb x_{\star}\right),\frac{\mb h}{\left\lVert \mb h\right\rVert _{2}}\rangle\right)\,,
\end{align*}
where in the second line the expectation with respect to the $\widetilde{f}_m$ is pulled outside of the supremum.
The next step is the standard \emph{symmetrization} argument \citep[see e.g.,][Lemma 2.3.1]{vanDerVaart-Weak-1996}.
Since $f_{m}$ and $\widetilde{f}_{m}$ are i.i.d., multiplying each
summand in the right-hand side of the inequality above by a corresponding
$\epsilon_{m}=\pm1$ does not change the distribution of the random
process and thereby the desired expected value. Take $\epsilon_{1},\epsilon_{2},\dotsc,\epsilon_{M}$
to be i.i.d. Rademacher random variables independent of everything
else. Therefore, we obtain 
\begin{align}
\E F_{\mc A_{\delta}}\left(f_{1},f_{2},\dotsc,f_{M}\right) & \le\E\sup_{\mb h\in\mc A_{\delta}}\left[\,\frac{1}{M}\sum_{m=1}^{M}\epsilon_{m}\psi_{\tau}\left(\langle\nabla\widetilde{f}_{m}\left(\mb x_{\star}\right),\frac{\mb h}{\left\lVert \mb h\right\rVert _{2}}\rangle\right)\right.\nonumber \\
 & \hspace{6em}\left.-\epsilon_{m}\psi_{\tau}\left(\langle\nabla f_{m}\left(\mb x_{\star}\right),\frac{\mb h}{\left\lVert \mb h\right\rVert _{2}}\rangle\right)\right]\nonumber \\
 & \le\E\sup_{\mb h\in\mc A_{\delta}}\,\frac{1}{M}\sum_{m=1}^{M}\epsilon_{m}\psi_{\tau}\left(\langle\nabla\widetilde{f}_{m}\left(\mb x_{\star}\right),\frac{\mb h}{\left\lVert \mb h\right\rVert _{2}}\rangle\right)\nonumber \\
 & \hspace{3ex}+\E\sup_{\mb h\in\mc A_{\delta}}\,\frac{1}{M}\sum_{m=1}^{M}-\epsilon_{m}\psi_{\tau}\left(\langle\nabla f_{m}\left(\mb x_{\star}\right),\frac{\mb h}{\left\lVert \mb h\right\rVert _{2}}\rangle\right)\nonumber \\
 & =2\,\E\sup_{\mb h\in\mc A_{\delta}}\,\frac{1}{M}\sum_{m=1}^{M}\epsilon_{m}\psi_{\tau}\left(\langle\nabla f_{m}\left(\mb x_{\star}\right),\frac{\mb h}{\left\lVert \mb h\right\rVert _{2}}\rangle\right)\,.\label{eq:Esup1}
\end{align}
Because $\psi_{\tau}\left(0\right)=0$ and $\psi_{\tau}\left(\cdot\right)$
is a contraction, we can invoke the \emph{Rademacher contraction
principle} \citep[Theorem 4.12]{Ledoux-Probability-2013} to show
that 
\begin{align*}
\E\sup_{\mb h\in\mc A_{\delta}}\,\frac{1}{M}\sum_{m=1}^{M}\epsilon_{m}\psi_{\tau}\left(\langle\nabla f_{m}\left(\mb x_{\star}\right),\frac{\mb h}{\left\lVert \mb h\right\rVert _{2}}\rangle\right) & \le\E\sup_{\mb h\in\mc A_{\delta}}\,\frac{1}{M}\sum_{m=1}^{M}\epsilon_{m}\langle\nabla f_{m}\left(\mb x_{\star}\right),\frac{\mb h}{\left\lVert \mb h\right\rVert _{2}}\rangle\\
 & =\frac{1}{\sqrt{M}}\,\mfk C_{M}\left(\mc A_{\delta}\right)\,,
\end{align*}
where $\mfk C_{M}\left(\mc A_{\delta}\right)$ is defined in \eqref{eq:complexity}
and can be interpreted as a measure of \emph{complexity} of $\mc A_{\delta}$
with respect to the law of $f$. It follows from \eqref{eq:Esup1}
that 
\begin{align*}
\E F_{\mc A_{\delta}}\left(f_{1},f_{2},\dotsc,f_{M}\right) & \le\frac{2}{\sqrt{M}}\,\mfk C_{M}\left(\mc A_{\delta}\right)\,.
\end{align*}
This bound together with \eqref{eq:BDI} guarantees that 
\begin{align*}
F_{\mc A_{\delta}}\left(f_{1},f_{2},\dotsc,f_{M}\right) & \le\frac{2}{\sqrt{M}}\,\mfk C_{M}\left(\mc A_{\delta}\right)+\frac{t\tau}{\sqrt{M}}\,,
\end{align*}
 with probability $\ge1-\exp(-2t^{2})$. Finally, on the same
event, it follows from \eqref{eq:penultimate-inequality} that for
all $\mb h\in\mc A_{\delta}$ we have
\begin{align*}
R_{M}^{\,+}\left(\mb x_{\star}+\mb h\right) & \ge\tau\left\lVert \mb h\right\rVert _{2}\P\left(\langle\nabla f\left(\mb x_{\star}\right),\mb h\rangle\ge\tau\left\lVert \mb h\right\rVert _{2}\right)-\frac{1}{M}\sum_{m=1}^{M}\left(\xi_{m}\right)_{+}\\
 & \hspace{3ex}-\frac{2\left\lVert \mb h\right\rVert _{2}}{\sqrt{M}}\,\mfk C_{M}\left(\mc A_{\delta}\right)-\frac{t\tau \left\lVert \mb h\right\rVert _{2}}{\sqrt{M}}\\
 & \ge\tau\left\lVert \mb h\right\rVert _{2}p_{\tau}\left(\mc A_{\delta}\right)-\frac{1}{M}\sum_{m=1}^{M}\left(\xi_{m}\right)_{+}-\frac{2\left\lVert \mb h\right\rVert _{2}}{\sqrt{M}}\,\mfk C_{M}\left(\mc A_{\delta}\right)-\frac{t\tau\left\lVert \mb h\right\rVert _{2}}{\sqrt{M}}\,.
\end{align*}
If $\mb x_{\star}+\mb h$ is feasible in \eqref{eq:anchored-regression},
then 
\begin{align*}
R_{M}^{\,+}\left(\mb x_{\star}+\mb h\right) & \le R_{M}^{\,+}\left(\mb x_{\star}\right)+\varepsilon\\
 & =\frac{1}{M}\sum_{m=1}^{M}\left(-\xi_{m}\right)_{+}+\varepsilon\,.
\end{align*}
Therefore, if $\widehat{\mb x}$ denotes a solution to \eqref{eq:anchored-regression}
and $\mb h=\widehat{\mb x}-\mb x_{\star}$, then we have 
\begin{align*}
\tau\left\lVert \mb h\right\rVert _{2}p_{\tau}\left(\mc A_{\delta}\right)-\frac{1}{M}\sum_{m=1}^{M}\left(\xi_{m}\right)_{+}\\
-\frac{2\left\lVert \mb h\right\rVert _{2}}{\sqrt{M}}\,\mfk C_{M}\left(\mc A_{\delta}\right)-\frac{t\tau\left\lVert \mb h\right\rVert _{2}}{\sqrt{M}} & \le\frac{1}{M}\sum_{m=1}^{M}\left(-\xi_{m}\right)_{+}+\varepsilon\,,
\end{align*}
 or equivalently
\begin{align*}
\left(\tau p_{\tau}\left(\mc A_{\delta}\right)-\frac{2\mfk C_{M}\left(\mc A_{\delta}\right)+t\tau}{\sqrt{M}}\right)\left\lVert \mb h\right\rVert _{2} & \le\frac{1}{M}\sum_{m=1}^{M}\left|\xi_{m}\right|+\varepsilon\,.
\end{align*}
 Applying the assumption \eqref{eq:sample-complexity} on the left-hand side completes the proof as
 \[\left\lVert\widehat{\mb{x}}-\mb{x}_\star\right\rVert_2 = \left\lVert\mb{h}\right\rVert_2\le \frac{2}{\tau p_\tau\left(\mc{A}_\delta\right)}\left(\frac{1}{M}\sum_{m=1}^{M}\left|\xi_{m}\right|+\varepsilon\right)\,.\]
 
\end{proof}
\bibliographystyle{abbrvnat}
\bibliography{references}

\end{document}

%% file: macros.tex
\global\long\def\mb#1{\hm{#1}}
\global\long\def\mbb#1{\mathbb{#1}}
\global\long\def\mc#1{\mathcal{#1}}

\global\long\def\mr#1{\mathrm{#1}}
\global\long\def\msf#1{\mathsf{#1}}
\global\long\def\mfk#1{\mathfrak{#1}}
\global\long\def\E{\mbb E}
\global\long\def\P{\mbb P}

\global\long\def\T{{\scalebox{0.55}{\ensuremath{\msf T}}}}

\global\long\def\sgn{\ensuremath{\,\mr{sgn}}}

\global\long\def\defeq{\stackrel{\textup{\tiny def}}{=}}
\global\long\def\bbone{\mbb 1}

\DeclareMathOperator*{\argmax}{\text{argmax}}

\DeclareMathOperator*{\maximize}{\text{maximize}}